\documentclass[twoside,11pt]{article}
\usepackage{pslatex}
\usepackage{jmlr2e}
\usepackage{amsmath,amssymb}
\usepackage{bbm}
\usepackage{times}
\usepackage[sort, numbers]{natbib}
\usepackage{tabls}
\usepackage{tabularx}
\usepackage{booktabs}
\usepackage{color}
\usepackage{url}
\usepackage{graphicx}

\usepackage[ruled]{strategy}
\usepackage{algorithmic}
\usepackage{graphicx}
\usepackage{epstopdf}

\usepackage[utf8]{inputenc}

\newtheorem{defi}{Definition}[section]

\renewenvironment{proof}[1][]{\par\noindent{\bf Proof #1\ }}{\hfill\BlackBox\\[2mm]}
\newenvironment{proof sketch}[1][]{\par\noindent{\bf Proof Sketch#1\ }}{\hfill\BlackBox\\[2mm]}

\newcommand{\hf}{\hat{f}}

\newcommand{\cV}{{\cal V}}

\newcommand{\E}{\mathbb{E}}

\newcommand{\F}{\mathcal{F}}

\newcommand{\mY}{\mathcal{Y}}

\newcommand{\ind}{\mathbbm{1}}

\newcommand{\cX}{{\cal X}}

\newcommand{\cF}{{\cal F}}

\newcommand{\cS}{{\cal S}}

\newcommand{\eqdef}{\triangleq}
\newcommand{\reals}{\mathbb{R}}

\newcommand{\polylog}{{\rm polylog}}

\newcommand{\PCS}{\rm PCS}
\newcommand{\LESS}{\rm LESS}
\newcommand{\iLESS}{{\rm{ILESS}}}
\newcommand{\ActiveiLESS}{{\textrm{Active-ILESS}}}
\newcommand{\BatchiLESS}{{\textrm{Batch-ILESS}}}

\newcommand{\argmin}{\mathop{\rm argmin}}

\newsavebox{\savepar}

	\firstpageno{1}
	
\begin{document}
	\bibliographystyle{ieeetr}

		\title{The Relationship Between Agnostic Selective Classification
			Active Learning and the Disagreement Coefficient}
		
		\author{%
			\name Roei Gelbhart
			\email roeige@cs.technion.ac.il\\
			\addr Department of Computer Science\\
			Technion -- Israel Institute of Technology\\
			\name Ran El-Yaniv
			\email rani@cs.technion.ac.il\\
			\addr Department of Computer Science\\
			Technion -- Israel Institute of Technology}

		\maketitle
		
		\begin{abstract}A selective classifier $(f,g)$ comprises a classification function $f$ and a binary selection function $g$,
			which determines if the classifier abstains from prediction, or uses $f$ to predict.
			The classifier is called pointwise-competitive if it classifies each point identically to the best classifier 
			in hindsight (from the same class), whenever it does not abstain.
			The quality of such a classifier is quantified by its rejection mass, defined to be the probability mass of the points it rejects. A ``fast'' rejection rate is achieved if the rejection mass is bounded from above 
			by $\tilde{O}(1/m)$ where $m$ is the number of labeled examples used to train the classifier 
			(and $\tilde{O}$ hides logarithmic factors).
			Pointwise-competitive selective (PCS) classifiers are intimately related to disagreement-based
			active learning and it is known that in the realizable case, a fast rejection rate 
			of a known PCS algorithm (called Consistent Selective Strategy) is equivalent to an exponential speedup of 
			the well-known CAL active algorithm.
			
			We focus on the agnostic setting, for which  there is a known algorithm called LESS that learns a 
			PCS classifier and achieves a fast rejection rate (depending on Hanneke’s disagreement coefficient) under strong assumptions. 
			We present an improved PCS learning algorithm called ILESS for which we show a fast rate
			(depending on Hanneke's disagreement coefficient) without any assumptions.
			Our rejection bound smoothly interpolates the realizable and agnostic settings.
			The main result of this paper is an equivalence between the following three
			entities: 
			(i) the existence of a fast rejection rate for any PCS learning algorithm (such as ILESS);
			(ii) a poly-logarithmic bound for Hanneke's disagreement coefficient;
			and 
			(iii) an exponential speedup for a new disagreement-based active learner called {\ActiveiLESS}.
		\end{abstract}

		\begin{keywords}
			Active learning, selective prediction, disagreement coefficient, selective sampling, selective classification, reject option, pointwise-competitive, selective classification, statistical learning theory, PAC learning, sample complexity, agnostic case
		\end{keywords}

		\section{Introduction}
		\label{sec:intro}

		\textbf{Selective classification} is a unique and extreme instance of the broader concept of
		confidence-rated prediction \citep{Chow70,VovkGS05,BartlettW08,yuan2010classification,cortes2016boosting,wiener2012pointwise,kocak2016conjugate,zhang2014beyond}.
		Given a training sample consisting of $m$ labeled instances,
		the learning algorithm is  required to output
		a \emph{selective classifier} \citep{ElYaniv10}, defined to be a pair $(f,g)$, where $f$ is a prediction function,
		chosen from some hypothesis class  $\cF$,
		and $g: \cX \to \{0,1\}$ is a \emph{selection function},
		serving as a qualifier for $f$ as follows: for any $x$, if  $g(x) =1$, the classifier predicts $f(x)$, and otherwise it abstains.
		The general performance of a selective classifier is quantified in terms of its \emph{coverage} and \emph{risk},
		where coverage is the probabilistic mass of non-rejected instances, and risk is
		the normalized average loss of $f$ restricted to non-rejected instances.
		Let $f^*$ be any (unknown) true risk minimizer\footnote{We assume that there exists an $f^{*}$ in $\cF$. Otherwise, we can 
			artificially define $f^*$ to be any function whose risk is sufficiently close to  $\inf_{f \in \cF}(R(f))$, for instance,  not greater than a small multiplicative factor from this infimum.} 
		in $\cF$ for the given problem.
		The  selective classifier $(f,g)$ is said to be \emph{pointwise-competitive} 
		if,  for each $x$ with  $g(x)=1$, it must hold that $f(x) =f^*(x)$ for all $f^* \in \cF$ \cite{WienerE14}. 
		Thus, pointwise-competitiveness w.h.p. over choices of the training sample, is a highly desirable property: it guarantees, for each non-rejected test point, the best possible classification obtainable using the best in-hindsight classifier from $\F$. We don't restrict $g$ to be from any specific hypothesis class, however, because we use disagreement based selective prediction, the selection of $\F$ will limit the possibilities of $g$. The scenario of a predefined decision functions hypothesis class is investigated in \cite{cortes2016learning}.  
		
		Pointwise-competitive selective classification ({\PCS}) was first considered in the realizable case
		\citep{ElYaniv10},  for which a simple consistent selective strategy (CSS) was shown to achieve a
		bounded and monotonically increasing (with $m$) coverage in various non-trivial settings.
		Note that in the realizable case, any {\PCS} strategy attains zero risk (over the sub-domain it covers).
		These results were recently extended to the agnostic setting \citep{WienerE14,wiener2011agnostic} with a related but different algorithm
		called \emph{low-error selective strategy (LESS)}, for which a number of coverage bounds were shown. These bounds relied on the fact that the underlying 
		probability distribution and the hypothesis class $\cF$ will satisfy  the so-called
		``$(\beta_1,\beta_2)$-Bernstein property'' \cite{BarMenPhi04}. 
		The coverage bounds in \citep{WienerE14,wiener2011agnostic} 
		are dependent on the parameters $\beta_1,\beta_2$. 
		This Bernstein property assumption (as presented in \cite{BarMenPhi04}),
		which allows for better concentration, 
		can be problematic.
		First, it is defined with respect to a unique true risk minimizer
		$f^*$, a property which is unlikely to hold in noisy agnostic settings.  
		Moreover, for arbitrary $\cF$, even for the 0/1 loss function, it is not 
		necessarily known whether the Bernstein property can hold at all.\footnote{It was mentioned in \cite{WienerE14} that, under the Tsybakov noise condition \cite{Tsybakov04}, the desired property holds, but this is guaranteed only for cases in which the Bayes classifier is within $\cF$, which is a fairly strong assumption in itself.} We removed the Bernstein assumption from our analysis.   
				
		Assuming that a selective classifier is w.h.p. pointwise-competitive, our key goal is a small rejection rate. We will say that a learner has a \textbf{fast $R^*$ rejection rate}, if w.h.p. the rejection rate is bounded by 
		$$
		\polylog(\frac{1}{R(f^*)+1/m}) \cdot R(f^*) + \frac{\polylog(m,d,1/\delta)}{m}.
		$$	
		Selective classification is very closely related to the field of \textbf{active learning (AL)}.
		In active learning, the learner can actively influence the learning process by selecting
		the points to be labeled. The incentive for introducing this
		extra flexibility is to reduce labeling efforts.
		A key question in theoretical studies of AL is how many
		label requests are sufficient to learn a given (unknown) target concept to a specified
		accuracy, a quantity called \emph{label complexity}. For an AL algorithm
		satisfying the ``passive example complexity'' property (consuming the same number
		of labeled/unlabeled examples as a passive algorithm for achieving
		the same error; see Definition \ref{passive example complexity}), we will say it has \textbf{$R^*$ exponential speedup}, if w.h.p. the number of labels it requests is bounded by
		$$
		\polylog(\frac{1}{R(f^*)+1/m}) \cdot R(f^*)m+ \polylog(m,d,1/\delta).
		$$ 
		
		The connection between AL and confidence-rated prediction is quite intuitive.
		A (pointwise-competitive) selective classifier $P$ can be straightforwardly used as the querying component of an active learning algorithm.
		This reduction is most naturally demonstrated in the stream-based AL model:
		at each iteration,
		the active algorithm trains a selective classifier on the currently available labeled samples, and then decides to query
		a newly introduced (unlabeled) point $x$ if $P$ abstains on $x$.


		Hanneke's \textbf{disagreement coefficient} \cite{Hanneke07} (see Definition \ref{disagreementCoefficient}), is a well-known parameter of the hypothesis class and the marginal distribution; it is used in most of the known label complexity bounds \cite{hsu:thesis, Hanneke07, ailon2011active}. The disagreement coefficient is the supremum of the relation between the disagreement mass of functions that are $r$-distanced from $f^*$ to $r$, over $r$. PCS classification is based on using generalization bounds to estimate the empirical error of $f^*$, and more specifically, its distance from the empirical error of the ERM. Whenever there is a unanimous agreement of all the functions that reside within a ball around the ERM, the classifier choses to classify. Thus, the abstain rate is dependent on the disagreement mass of the functions within the ball. The radius of the ball depends on the generalization bounds. The generalization bounds we use are of the form $\tilde{O}(1/m)$ for the realizable case (we consider the realizable case here for simplicity). After observing $m$ examples, we can bound the disagreement mass of a ball around the ERM, by multiplying the radius of the ball, which is $\tilde{O}(1/m)$, with the disagreement coefficient. Thus, if for example, the disagreement coefficient is bounded by a constant, the abstain rate of some PCS algorithms can be bounded by $\tilde{O}(1/m)$ for the realizable case. This 
		gives a basic idea of the disagreement coefficient, which will be formally presented later on.	
		
		Note that in principle, the disagreement coefficient can 
		be replaced by another important quantity, namely,
		the \textbf{version space compression set size},
		recently shown to be equivalent to it \cite{JMLR:v16:wiener15a,el2015version}. Specifically, an $O(\polylog(m)log(1/\delta))$ version space compression set size minimal bound was shown in \cite[Corollary 11]{JMLR:v16:wiener15a}, to be equivalent to an $O(\polylog(1/r))$ disagreement coefficient.

		The first contribution of this paper is a novel selective classifier, called {\iLESS}, which utilizes a tighter generalization error bound than 
		LESS and depends on $R(f^*)$ (and interpolates the agnostic and realizable cases). Most importantly,
		the new strategy can be analyzed completely without the Bernstein condition. 

		We derive an active learning algorithm, called {\ActiveiLESS}, corresponding to our selective classifier, {\iLESS}. {\ActiveiLESS} is constructed to work in a stream-based AL model and its querying function is extremely conservative: for each unlabeled example, the algorithm requests its label if and only if the labeling of the optimal classifier (from the same class) on this point cannot be inferred from information already acquired.
		This querying strategy, which is often termed ``disagreement-based,'' has been used in
		a number of stream-based AL algorithms such as
		$A^2$ (\emph{Agnostic Active}), developed in  \cite{balcan2006agnostic},
		RobustCAL, studied by the authors of \cite{Hanneke11,Hanneke_book} and \cite{hanneke:12b},
		or the general agnostic AL algorithm of \cite{dasgupta2007general}. In \cite{NIPS2015_5939}, a computationally efficient algorithm for disagreement based AL.
		
		The first formal relationship between {\PCS} classification and AL was proposed in \citep{el2012active,Wiener13},
		where the aforementioned CSS algorithm was shown to be equivalent to the well-known CAL AL algorithm of
		\cite{CohAtlLad94}, in the sense that a fast coverage rate for CSS was proven to be equivalent to exponential label complexity
		speedup for CAL. This result applies to the realizable setting only.
		Our first contribution is a similar equivalence relation between pointwise-competitive selective classification and AL, which applies to the more challenging agnostic case and
		smoothly interpolates the realizable and agnostic settings.

		Our second and main contribution is to show a complete equivalence between 
		(i) selective classification with a fast $R^*$ rejection rate, 
		(ii) AL with $R^*$ exponential speedup (represented by {\ActiveiLESS}), and 
		(iii) the existence of an $f^*$ with a bounded disagreement coefficient.
		This is illustrated in Figure \ref{figure1}, where the blue errors indicate the equivalence relationships we prove in this paper, and the red arrow indicates a previously known result (from \cite{hsu:thesis, Hanneke07}) (and can also be deduced from the other arrows).

		\begin{figure}[htb]	
				\includegraphics[scale=0.4]{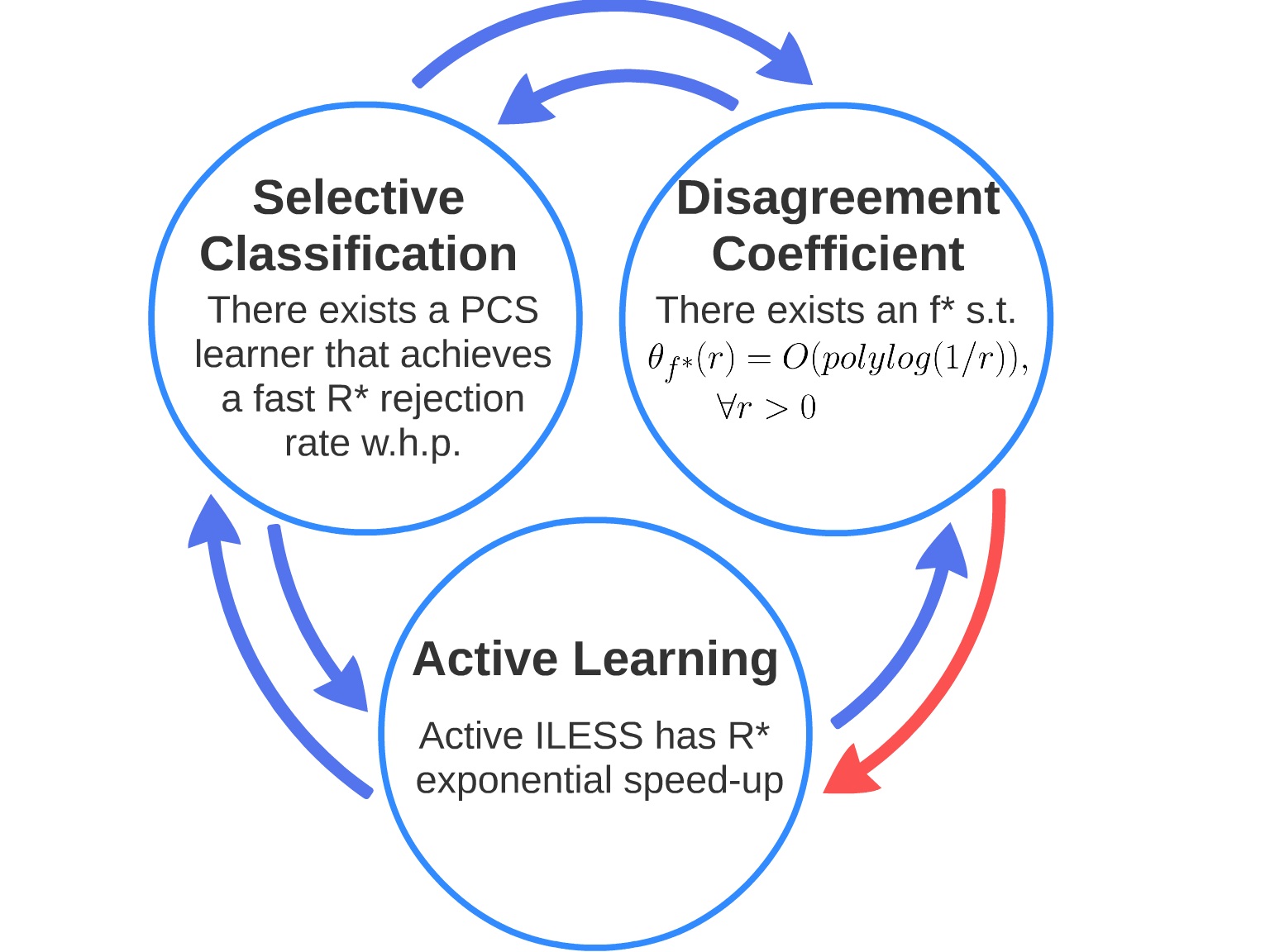}
				\centering
				    \caption{Main results \label{figure1}}
		\end{figure}
		
		\section{Definitions}
		\label{sec:definitions}
		Consider a domain  $\mathcal{X}$, and a binary label set $\mathcal{Y} = \{\pm 1\}$.
		A learning problem is specified via a hypothesis class $\mathcal{F}$ and an unknown probability distribution
		$\mathcal{P_{X,Y}}$. Given a sequence of labeled training examples $S_{m}=((x_{1},y_{1}),(x_{2},y_{2}),...,(x_{m},y_{m}))$,  such that $\forall i,(x_{i},y_{i})\in \mathcal{X}\times \mathcal{Y}$, the empirical error of a hypothesis $f$ over  $S_m$  is
		$\hat{R}(f,S_{m}) \eqdef \frac{1}{m}\sum_{i=1}^m \ell(f(x_i), y_i)$,
		where $\ell: \mY \times \mY \to \reals^+$ is a loss function. 
		In this paper we will mainly focus on the zero-one loss function, $\ell_{01}(y,y') \eqdef \ind\{y \neq y'\}$.
		The true (zero-one) error of $f$ is
		$R(f) \eqdef \E_{\mathcal{P}}\left[\ell_{01} (f(x),y)  \right]$.
		An empirical risk minimizer hypothesis (henceforth an ERM) is
		\begin{equation}
		\hf(S_{m}) \eqdef \argmin_{f \in \cF} \hat{R}(f,S_{m}) ,
		\end{equation}
		and a true risk minimizer is
		$f^* \eqdef \argmin_{f \in \cF} R(f)$.\footnote{We assume that $f^{*}$ exists, and that it need not be unique, in which case $f^{*}$  refers to any one of the minimizers.}
		
		We acquire the following definitions from \cite{WienerE14}.
		For any hypothesis class $\cF$, hypothesis $f \in \cF$, distribution $\mathcal{P_{X,Y}}$, sample $S_m$, and real number $r>0$, define
		the true and empirical \emph{low-error sets},
		\begin{equation}
		\label{eq:V}
		\cV(f,r) \eqdef \left\{f' \in \cF : R(f') \leq R(f) + r \right\}
		\end{equation}
		and
		\begin{equation}
		\label{eq:Vhat}
		\hat{\cV}(f,r) \eqdef \left\{f' \in \cF : \hat{R}(f',S_{m}) \leq \hat{R}(f,S_{m}) + r \right\}.
		\end{equation}
		Let $G \subseteq \cF$. The \emph{disagreement set} \cite{Hanneke07}
		and \emph{agreement set} \cite{ElYaniv10} w.r.t. $G$ are defined, respectively, as
		\begin{equation}
		\label{eq:dis}
		DIS(G) \eqdef \left\{x \in \cX : \exists f_1,f_2 \in G \quad \text{s.t.} \quad f_1(x) \neq f_2(x)\right\}
		\end{equation}
		\begin{equation}
		\label{eq:agr}
		\text{and} \ \ AGR(G) \eqdef \left\{x \in \cX : \forall  f_1,f_2 \in G \quad \text{s.t.} \quad f_1(x) =  f_2(x)\right\}.
		\end{equation}	
		In \emph{selective classification} \cite{ElYaniv10}, the learning algorithm receives $S_m$ and is 
		required to output
		a \emph{selective classifier}, defined to be a pair $(f,g)$, where $f \in \cF$ is a classifier, 
		and $g: \cX \to \{0,1\}$ is a \emph{selection function},
		serving as a qualifier for $f$ as follows.
		For any $x \in \cX$, $(f,g)(x) = f(x)$ iff $g(x) =1$. Otherwise, the classifier outputs ``I don't know''. 
		For any selective classifier $(f,g)$ we define its coverage to be
		$$
		\Phi(f,g) \eqdef \Pr_{X \sim \mathcal{P_{X}}}(g(X)=1),
		$$ and its complement, $1- \Phi$,  is called the \textbf{abstain rate}.  
		For any $f \in \cF$ and $r > 0$, define the set
		$B(f,r)$ of all hypotheses that reside within a ball of radius $r$ around $f$,
		$$
		B(f,r) \eqdef \left\{f' \in \cF : \Pr_{X \sim \mathcal{P_{X}}}\left\{f'(X) \neq f(X)\right\} \leq r \right\}.
		$$
		For any $G \subseteq \cF$, and distribution $\mathcal{P_X}$, we denote by $\Delta G$ the volume of the disagreement set 
		of $G$ (see (\ref{eq:dis})),
		$\Delta G \eqdef \Pr\left\{DIS(G)\right\}$.
		\begin{defi} [Disagreement Coefficient]	\label{disagreementCoefficient}
		Let $r_0 \geq 0$. Then,
		Hanneke's \emph{disagreement coefficient} \cite{Hanneke07} 
		of a classifier $f \in \cF$ with respect to the target distribution $\mathcal{P_X}$ is		
		\begin{equation}
		\label{eq:disagreementCoefficient1}
		\theta_{f}(r_0)\eqdef\sup_{r>r_0}\frac{\Delta B(f,r)}{r},
		\end{equation}		
		and the general \emph{disagreement coefficient} of the entire hypothesis class $\cF$ is
		\begin{equation}
		\label{eq:disagreementCoefficient}
		\theta(r_0) \eqdef \sup_{f \in \cF}\theta_{f}(r_0).
		\end{equation}
		\end{defi}
		Notice that this definition of the disagreement coefficient is independent of $\mathcal{P_{Y|X}}$.
		Another commonly used definition of the disagreement coefficient does depend on a true risk minimizer
		$f^*$, as follows:
				\begin{equation}
				\label{eq:disagreementCoefficient2}
				\theta'(r_0) = \sup_{r>r_0}\frac{\Delta B(f^*,r)}{r}.
				\end{equation}
		Clearly, it always holds that $\theta' \leq \theta$. The independence of $\theta$ of unknown quantities
		such as the underlying distribution (and $f^*$), however, is a convenient
		property that sometimes allows for a direct estimation of $\theta$, which only depends on the marginal distribution, $\mathcal{P_X}$.
		This is, for example, the case in active learning, where labels are expensive but information about the marginal distribution (provided by unlabeled examples) is cheap. 
		Note also that the above definition of $\theta'$ implicitly assumes a unique $f^*$. Nevertheless, the definition can be extended to cases where
		$f^*$ is not unique, in which case the infimum over all $f^*$ can be considered (the analysis can be extended accordingly using limits). For more on the disagreement coefficient, and examples of probabilities distributions and hypothesis classes for which it is bounded, see \cite{Hanneke_book}.

		\section{Convergence Bounds and LESS}
		We use a uniform convergence bound from \cite{dasgupta2007general,BousquetBL03}.
		Define  convergence slacks  $\sigma_{R-\hat{R}}(m,\delta,d,R,\hat{R})$ and $\sigma_{\hat{R}-R}(m,\delta,d,R,\hat{R})$, given in terms of the training sample, $S_{m}$, its size, $m$, the confidence parameter, $\delta$, and the VC-dimension $d$ of the class $\cF$. For any $f \in \cF$,

		\begin{equation}
			\label{eq:slackOne}
		\sigma_{R-\hat{R}}(m,\delta,d,R,\hat{R}) \eqdef \min \{\underbrace{\frac{4d \ln( \frac{16me}{d\delta} )}{m} +  \sqrt{\frac{4d \ln( \frac{16me}{d\delta} )}{m} \cdot \hat{R}}}_{\hat{\sigma}_{R-\hat{R}}(m,\delta,d,\hat{R})} ,\underbrace{ \sqrt{\frac{4d \ln( \frac{16me}{d\delta} )}{m} \cdot R}}_{\bar{\sigma}_{R-\hat{R}}(m,\delta,d,R)}  \}
		\end{equation}
		and
		\begin{equation}
		\sigma_{\hat{R}-R}(m,\delta,d,R,\hat{R}) \eqdef \min \{\underbrace{ \frac{4d \ln( \frac{16me}{d\delta} )}{m} +  \sqrt{\frac{4d \ln( \frac{16me}{d\delta} )}{m} \cdot R}}_{\bar{\sigma}_{\hat{R}-R}(m,\delta,d,R)},\underbrace{ \sqrt{\frac{4d \ln( \frac{16me}{d\delta} )}{m} \cdot \hat{R}}}_{\hat{\sigma}_{\hat{R}-R}(m,\delta,d,\hat{R})}  \}.
		\end{equation}
		To simplify the analysis, we further decompose the above  slack terms to their empirical and non-empirical components. For~(\ref{eq:slackOne}), we thus have,
		respectively,
		\begin{equation}
		\hat{\sigma}_{R-\hat{R}}(m,\delta,d,\hat{R}) \eqdef \frac{4d \ln( \frac{16me}{d\delta} )}{m} +  \sqrt{\frac{4d \ln( \frac{16me}{d\delta} )}{m} \cdot \hat{R}}
		\end{equation}
		and
		\begin{equation}
		\hat{\sigma}_{\hat{R}-R}(m,\delta,d,\hat{R}) \eqdef \sqrt{\frac{4d \ln( \frac{16me}{d\delta} )}{m} \cdot \hat{R}} .
		\end{equation} 
		Similarly, the non-empirical part in these minimums are denoted by
		$\bar{\sigma}_{R-\hat{R}}$ and $\bar{\sigma}_{\hat{R}-R}$. 
		With this notation, we can write, for example, 
			$\sigma_{R-\hat{R}}  = \min \{ 
			\hat{\sigma}_{R-\hat{R}}
			,\bar{\sigma}_{R-\hat{R}}  \}$.
		Our Lemma \ref{lemma:agnostic} is taken from \cite[Lemma 1]{dasgupta2007general}, which is based on \cite[Theorem 7]{BousquetBL03} \footnote{In the original lemma from \cite{dasgupta2007general}, there appears $S(\mathcal{H},n)$, the growth function. We plug in Sauer's Lemma, $S(\mathcal{H},n)\leq (\frac{em}{d})^d$, into Lemma 1 from \cite{dasgupta2007general} to get our lemma.}.
		\begin{lemma} [\cite{dasgupta2007general}]
			\label{lemma:agnostic}
			Let $\mathcal{F}$ be a hypothesis class with VC-dimension $d$.
			For any $0<\delta < 1$,
			with probability of at least $1-\delta$ over the
			choice of $S_m$ from $\mathcal{P}^m$, any hypothesis $f \in \cF$ satisfies
			\begin{equation}
			R(f) \leq \hat{R}(f) + \sigma_{R-\hat{R}}\left(m,\delta,d,R(f),\hat{R}(f)\right)	\label{eq:bound1}
			\end{equation} 
			\begin{equation}
			\hat{R}(f) \leq R(f) + \sigma_{\hat{R}-R}\left(m,\delta,d,R(f),\hat{R}(f)\right).	\label{eq:bound2}
			\end{equation}

		\end{lemma}
		
		Strategy~\ref{alg:SemanticSort} is the $\LESS$ algorithm of \cite{WienerE14}.
		$\LESS$ learns w.h.p. a pointwise-competitive selective classifier, $(f,g)$, where $f \in \F$ and
		$g: \cX \to \{0,1\}$ is its selection function which determines whether to abstain or to classify.
		A \emph{pointwise-competitive selective classifier} must satisfy the following condition: for each $x$ with  $g(x)=1$, it must hold that $f(x) =f^*(x)$ for all $f^* \in \cF$.
		\begin{remark}
		The original definition of pointwise-competitiveness from \cite{WienerE14} requires a single $f^*$. We widen the definition to cases for which there are more than one $f^*$, and require that a pointwise-competitive selective classifier will be equal to \emph{all} $f^*$, wherever $g=1$. This extrapolation seems a bit strict. However, even if the requirement would have been relaxed to ``any $f^*$'', any pointwise-competitive selective classifier would still have been forced to identify with all $f^*$, as it is impossible to differentiate whether a set of functions are all $f^*$, or one is better than the rest.   
		\end{remark}	
		
		The main idea behind $\LESS$ is  that, w.h.p. all $f^{*}$ lie within a ball around an ERM
		hypothesis with error radius of $2\sigma(m,\delta / 4,d)$, where 
		\begin{equation}
		\label{eq:sigma}
		\sigma(m,\delta,d) \eqdef 2\sqrt{\frac{2d\left(\ln\frac{2me}{d}\right)+\ln{\frac{2}{\delta}}}{m}}
		\end{equation}
		is the slack term of a certain uniform convergence bound. 
		Therefore, if all the functions in that ball agree over the labeling of any instance $x$, we know with high probability that all $f^{*}$ label $x$ the same way as the ERM.
		This property ensures that $\LESS$ is pointwise-competitive w.h.p.

		\begin{strategy}
			\caption{Agnostic low-error selective strategy (\LESS)}
			\label{alg:one}
			\footnotesize {
				\begin{algorithmic}[1]
					\REQUIRE  Sample set of size $m$, $S_m$,\\
					Confidence level $\delta$\\
					Hypothesis class $\mathcal{F}$ with VC dimension $d$\\
					\ENSURE A selective classifier $(h,g)$ 
					\STATE Set $\hat{f} = ERM(\cF,S_m)$, i.e., $\hat{f}$ is any empirical risk minimizer from $\cF$
					\STATE Set $G = \hat{\cV}\left(\hat{f}, 2\sigma(m,\delta / 4,d)  \right)$
					\STATE Construct $g$ such that $g(x) = 1 \Longleftrightarrow x \in \left\{\cX \setminus DIS\left(G\right)\right\}$
					\STATE $f = \hat{f}$
				\end{algorithmic}
			}
			\label{alg:SemanticSort}
		\end{strategy}

	\section{ILESS}	\label{sec:iLESS}
		
	\begin{strategy}
		\caption{Improved Low-Error Selective Strategy (\iLESS)}
		\label{alg:iLESS}
		\footnotesize {
			\begin{algorithmic}[1]
				\REQUIRE  Sample set of size $m$, $S_m$,\\
				Confidence level $\delta$\\
				Hypothesis class $\mathcal{F}$ with VC dimension $d$\\
				\ENSURE A selective classifier $(h,g)$ 
				\STATE Set $\hat{f} = ERM(\cF,S_m)$, i.e., $\hat{f}$ is any empirical risk minimizer from $\cF$
				\STATE Set $\sigma_{\iLESS}= \hat{\sigma}_{R-\hat{R}}
				\left(m,\delta,d,\hat{R}(\hat{f},S_m)
				\right) + \bar{\sigma}_{\hat{R}-R}
				\left(m,\delta,d,\hat{R}(\hat{f},S_m)+\hat{\sigma}_{R-\hat{R}}(m,\delta,d,\hat{R}(\hat{f},S_m))
				\right)$
				\STATE Set $G = \hat{\cV}\left(\hat{f}, \sigma_{\iLESS}  \right)$
				\STATE Construct $g$ such that $g(x) = 1 \Longleftrightarrow x \in \left\{\cX \setminus DIS\left(G\right)\right\}$
				\STATE $h = \hat{f}$
			\end{algorithmic}
		}
		\label{alg:LESS-sigma}
	\end{strategy}
	
		In this section we introduce an improved version of $\LESS$, called $\iLESS$,
		which utilizes a radius of the form $\polylog(m,1/\delta ,d)\cdot(\frac{1}{m}+\sqrt{\frac{R(f^{*})}{m}})$. 
		Noting that the radius, $2\sigma(m,\delta / 4,d)$, used by $\LESS$ to define $G = \hat{\cV}$, 
		is of the form  $\polylog(m,1/\delta ,d) / \sqrt{m}$, we observe that 
		in cases where $R(f^{*}) \approx \frac{C}{m}$, this new radius  behaves as $\frac{\polylog(m,1/\delta ,d)}{m}$. We later show that this radius allows $\iLESS$ to achieve a faster
		rejection decay rate than the one achieved by $\LESS$.

		Consider the pseuodo-code of $\iLESS$ given in Strategy~\ref{alg:iLESS}.
		We now analyze $\iLESS$, and begin 
		by showing in Lemma \ref{thm:f*in_LESS} that $\iLESS$ is pointwise-competitive w.h.p., i.e., for any $x$ for which $g(x) = 1$, $f(x) = f^*(x)$ for all $f^*$. The calculation of $g$ appears to be very problematic, as for a specific $x$, a unanimous decision over an infinite number of functions must be ensured. This problem was shown to be reducible to finding an ERM under one constraint (\cite[Lemma 6.1]{el2011agnostic} a.k.a. the disbelief principle). This is a difficult problem, nonetheless, albeit one that could be estimated with heuristics.

		\begin{defi} \label{E}
			Let $\mathcal{F}$ be a hypothesis class with a finite VC dimension $d$, and let $\mathcal{P_{X,Y}}$ be an unknown probability distribution. Given a sample set $S_m$, drawn from $\mathcal{P_{X,Y}}$, we denote by $\cal E$ the event where both inequalities (\ref{eq:bound1}) and (\ref{eq:bound2}) of Lemma~\ref{lemma:agnostic} simultaneously hold. We know from the lemma that $\cal E$ occurs with probability of at least $1-\delta$.
		\end{defi}
		\begin{lemma}[$\iLESS$ is pointwise-competitive]
			\label{thm:f*in_LESS}
			Given that event $\cal E$ occurred (see Definition \ref{E}),
			for all $f^* \in \mathcal{F}$, $f^*$ resides within $G$ (from Strategy~\ref{alg:LESS-sigma}), and therefore, $\iLESS$ is pointwise-competitive w.h.p.
		\end{lemma}
		
		\begin{proof}
			From (\ref{eq:bound2}) it follows that,
			\begin{eqnarray}
			\label{u1}
			\hat{R}(f^{*},S_m) &\leq& R(f^{*})+\sigma_{\hat{R}-R}(m,\delta,d,R(f^{*}),\hat{R}(f^{*},S_m)) \nonumber \\
			&\leq& R(f^{*})+\bar{\sigma}_{\hat{R}-R}(m,\delta,d,R(f^{*})).
			\end{eqnarray}
			Additionally, by the definition of $f^{*}$, we know that it has the lowest true error, and using Inequality~(\ref{eq:bound1}) from Lemma~\ref{lemma:agnostic} we obtain,
			\begin{eqnarray}
			R(f^{*}) &\leq& R(\hat{f}) \nonumber \\
			&\leq& \hat{R}(\hat{f},S_m) +\sigma_{R-\hat{R}}(m,\delta,d,R(\hat{f}),\hat{R}(\hat{f},S_m)) \nonumber \\
			&\leq&\hat{R}(\hat{f},S_m) +\hat{\sigma}_{R-\hat{R}}(m,\delta,d,\hat{R}(\hat{f},S_m)). 			\label{u2} 
			\end{eqnarray}
			Finally, by applying (\ref{u2}) in (\ref{u1}), we have,

			\begin{multline}
			\hat{R}(f^{*},S_m) \leq  \hat{R}(\hat{f},S_m) +\hat{\sigma}_{R-\hat{R}}(m,\delta,d,\hat{R}(\hat{f},S_m))
			  \\+\bar{\sigma}_{\hat{R}-R}\left(m,\delta,d,\hat{R}(\hat{f},S_m) +\hat{\sigma}_{R-\hat{R}}(m,\delta,d,\hat{R}(\hat{f},S_m))\right) \nonumber ,
			\end{multline}

			\begin{eqnarray}
			\hat{R}(f^{*},S_m) &\leq&  \hat{R}(\hat{f},S_m) +\hat{\sigma}_{R-\hat{R}}(m,\delta,d,\hat{R}(\hat{f},S_m))
			+\bar{\sigma}_{\hat{R}-R}\left(m,\delta,d,\hat{R}(\hat{f},S_m) +\hat{\sigma}_{R-\hat{R}}(m,\delta,d,\hat{R}(\hat{f},S_m))\right) \nonumber ,
			\end{eqnarray}
			which means that $ f^{*}\in G$.
		\end{proof}
		
		Lemma~\ref{lemma:radius} below
		bounds the radius $\sigma_{\iLESS}$ of $\iLESS$. 
		The lemma utilizes the notation
		$$
			A \eqdef 4d \ln( \frac{16me}{d\delta} ),
		$$
		with which, by the definition of $\sigma_{\iLESS}$ (see Strategy~\ref{alg:iLESS}),
		we have,
		\begin{eqnarray}
		\sigma_{\iLESS} &=& \hat{\sigma}_{R-\hat{R}}(m,\delta,d,\hat{R}(\hat{f},S_m))
		+\bar{\sigma}_{\hat{R}-R}\left(m,\delta,d,\hat{R}(\hat{f},S_m) +\hat{\sigma}_{R-\hat{R}}(m,\delta,d,\hat{R}(\hat{f},S_m))\right) \nonumber \\
		&=& \frac{A}{m} + \sqrt{\frac{A}{m} \cdot \hat{R}(\hat{f},S_m)} + \frac{A}{m} + \sqrt{\frac{A}{m} \cdot [\hat{R}(\hat{f},S_m) + \frac{A}{m} + \sqrt{\frac{A}{m} \cdot \hat{R}(\hat{f},S_m)} ]} . \nonumber \\
		\label{t5}
		\end{eqnarray}
		
		\begin{lemma}
			\label{lemma:radius}
			Given that event $\cal E$ (see Definition \ref{E}) occurred, the radius of $\iLESS$ satisfies
			\begin{eqnarray}
			\sigma_{\iLESS}=6\frac{A}{m} + 3\sqrt{\frac{A}{m} \cdot R(f^{*})}=O(\frac{A}{m}+ \sqrt{\frac{A}{m} \cdot R(f^{*})} ),
			\end{eqnarray}	
			where
			\begin{math}
				A \eqdef 4d \ln( \frac{16me}{d\delta} ).
			\end{math}		 
		\end{lemma}
		\begin{proof}
		Under our assumption, inequalities (\ref{eq:bound1}) and (\ref{eq:bound2}) hold for every $f \in \cF$. We thus have 
		\begin{equation}
		\hat{R}(\hat{f},S_m) \leq \hat{R}(f^*,S_m) \leq R(f^{*})+ \frac{A}{m} +  \sqrt{\frac{A}{m} \cdot R(f^{*})}.
		\label{t6}
		\end{equation}
		Replacing the three occurrences of $\hat{R}(f^*,S_m)$ in~(\ref{t5}) with the R.H.S. of (\ref{t6}),  
		and using the basic inequalities 
		$\sqrt{A+B} \leq \sqrt{A} + \sqrt{B}$ and $\sqrt{AB} \leq A/2 + B/2$,
		 we get,
			\begin{eqnarray}
			\sigma_{\iLESS} &\leq& \frac{A}{m} + \sqrt{\frac{A}{m} \cdot\left(R(f^{*})+ \frac{A}{m} +  \sqrt{\frac{A}{m} \cdot R(f^{*})}\right)} + \frac{A}{m} + \nonumber \\ &&+  \sqrt{\frac{A}{m} \cdot 
				\left[{R(f^{*})+ \frac{A}{m} +  \sqrt{\frac{A}{m} \cdot R(f^{*})}} + \frac{A}{m} + \sqrt{\frac{A}{m} \cdot \left(R(f^{*})+ \frac{A}{m} +  \sqrt{\frac{A}{m} \cdot R(f^{*})}\right)} 
				\right]} \nonumber \\
			&\leq& \frac{A}{m} + \sqrt{\frac{A}{m} \cdot\left(R(f^{*})+ \frac{A}{m} +  \frac{A}{2m} + \frac{1}{2}R(f^{*})\right)} + \frac{A}{m} + \nonumber \\ &&+  \sqrt{\frac{A}{m} \cdot 
				 \left[R(f^{*})+ \frac{A}{m} +  \frac{A}{2m} + \frac{1}{2}R(f^{*}) + \frac{A}{m} + \sqrt{\frac{A}{m} \cdot \left(R(f^{*})+ \frac{A}{m} +  \frac{A}{2m} + \frac{1}{2}R(f^{*})\right)} \right]
				 } \nonumber \\
			&\leq& \frac{2A}{m} + \frac{3A}{2m} + \frac{3}{2}\sqrt{\frac{A}{m} \cdot R(f^{*})}+  \sqrt{\frac{A}{m} \cdot 
				\left[ \frac{5A}{2m} + \frac{3}{2}R(f^{*}) + \sqrt{\frac{A}{m} \cdot \left(\frac{3A}{2m} + \frac{3}{2}R(f^{*})\right)} \right]
				} \nonumber \\		
			&\leq&  \frac{7A}{2m} + \frac{3}{2}\sqrt{\frac{A}{m} \cdot R(f^{*})}+  \sqrt{\frac{A}{m} \cdot 
				\left[ \frac{5A}{2m} + \frac{3}{2}R(f^{*}) + \frac{3A}{2m} +\sqrt{\frac{A}{m} \cdot \frac{3}{2}R(f^{*})} \right]
				} \nonumber \\
			&\leq& \frac{7A}{2m} + \frac{3}{2}\sqrt{\frac{A}{m} \cdot R(f^{*})}+  \sqrt{\frac{A}{m} \cdot 
				\left[ \frac{5A}{2m} + \frac{3}{2}R(f^{*}) + \frac{3A}{2m} +\frac{3A}{4m} + \frac{3}{4}R(f^{*}) 
				\right]} \nonumber \\
			&\leq& \frac{7A}{2m} + \frac{3}{2}\sqrt{\frac{A}{m} \cdot R(f^{*})}+ \sqrt{\frac{19}{4}}\frac{A}{m}+ \sqrt{\frac{A}{m} \cdot\frac{9}{4}R(f^{*})} \nonumber \\
			&\leq& 6\frac{A}{m} + 3\sqrt{\frac{A}{m} \cdot R(f^{*})} \label{t7}.
			\end{eqnarray}	
\end{proof}
		In comparison, the radius of
		$\LESS$ is of order $O(\sqrt{\frac{A}{m}})$, which can be significantly larger when $R(f^{*})$ is small.
		This potential radius advantage translates to a potential coverage advantage of $\iLESS$, as stated in the
		following theorem. 
		
		\begin{theorem}  
			\label{thm:LessRejection}
			Let $\mathcal{F}$ be a hypothesis class with a finite VC dimension $d$, and let	
			$\mathcal{P_{X,Y}}$ be an unknown probability distribution. Given that event $\cal E$ (see Definition \ref{E}) occurred, for all $f^*$, the abstain rate is bounded by
			$$
			1-\Phi(\iLESS) \leq \theta_{f^*}(R_{0}) \cdot R_{0},	
			$$		
			where
			$$
			R_{0} \eqdef 2\cdot R(f^{*}) + 11 \cdot \frac{A}{m} + 6 \cdot \sqrt{\frac{A}{m} \cdot R(f^{*})}.
			$$
			This immediately implies (by definition) that 
			$$
			1-\Phi(\iLESS) \leq \theta(R_{0}) \cdot R_{0}.	
			$$	
		\end{theorem} 
		\begin{remark}
			Note that $R_0=O(R(f^{*}) + \frac{A}{m} )$ due to $\sqrt{\frac{A}{m} \cdot R(f^{*})}\leq\frac{1}{2}(\frac{A}{m}+R(f^{*}))$.
		\end{remark}	
			
		\begin{proof}
			We start by showing that $G$, defined in Strategy \ref{alg:iLESS}, resides within a ball around any specific $f^{*}$. To do so, we need to bound the true error of all functions in $G$. 
			\begin{eqnarray}
			f \in G &\Rightarrow& \hat{R}(f,S_m)\leq\hat{R}(\hat{f},S_m) + \sigma_{\iLESS}   \label{e18}\\
			&\Rightarrow& \hat{R}(f,S_m) \leq R(f^{*})+ \frac{A}{m} +  \sqrt{\frac{A}{m} \cdot R(f^{*})} + 6\frac{A}{m} + 3\sqrt{\frac{A}{m} \cdot R(f^{*})} \label{e19}\\
			&\Rightarrow& \hat{R}(f,S_m) \leq R(f^{*})+7\cdot \frac{A}{m}+4\cdot \sqrt{\frac{A}{m} \cdot R(f^{*})} \label{e20},
			\end{eqnarray}
			where inequality (\ref{e18}) is by the definition of $G$, and inequality~(\ref{e19}) follows from (\ref{t6}) and (\ref{t7}) (under event ${\cal E}$).
			We then have,
			\begin{eqnarray}
			R(f) &\leq& \hat{R}(f,S_m) + \hat{\sigma}_{R-\hat{R}}(m,\delta,d,\hat{R}) \label{e21}  \\
			&\leq& \hat{R}(f,S_m) + \frac{A}{m} + \sqrt{\frac{A}{m} \cdot \hat{R}(f,S_m)} \label{e22}\\
			&\leq& R(f^{*})+8\cdot \frac{A}{m}+4\cdot \sqrt{\frac{A}{m} \cdot R(f^{*})}+ \sqrt{\frac{A}{m}\cdot
			\left[
				  R(f^{*})+7\cdot \frac{A}{m}+4\cdot \sqrt{\frac{A}{m} \cdot R(f^{*})}
			\right]}  \label{e23}\\
			&\leq& R(f^{*})+8\cdot \frac{A}{m}+4\cdot \sqrt{\frac{A}{m} \cdot R(f^{*})}+ \sqrt{\frac{A}{m}\cdot
				\left[
				 3R(f^{*})+9\cdot \frac{A}{m}
				\right]}  \label{e24}\\
			&\leq& R(f^{*})+11\cdot \frac{A}{m}+6\cdot \sqrt{\frac{A}{m} \cdot R(f^{*})},  \label{e25}
			\end{eqnarray}
			where inequality~(\ref{e21}) is (\ref{eq:bound1}) (which holds given $\cal E$),
			inequality~(\ref{e22}) follows directly from the definition of $\hat{\sigma}_{R-\hat{R}}$,
			inequality~(\ref{e23}) is obtained using (\ref{e20}), inequality (\ref{e24}) follows from 
			$\sqrt{AB} \leq A/2 + B/2$, and~(\ref{e25}) from $\sqrt{A+B} \leq \sqrt{A} + \sqrt{B}$.
			
			Using (\ref{e25}), for all $f \in G$, and any $f^*$ we have,
			\begin{eqnarray}
				\Pr_{X \sim \mathcal{P_{X}}}\left\{f(X) \neq f^{*}(X)\right\} &=& \Pr_{X,Y \sim \mathcal{P_{X,Y}}}\left\{f(X) \neq f^{*}(X) \wedge f^{*}(X)=Y \right\} + \Pr_{X,Y \sim \mathcal{P_{X,Y}}}\left\{f(X) \neq f^{*}(X) \wedge f^{*}(X)\neq Y \right\}  \nonumber \\
				&\leq& \Pr_{X,Y \sim \mathcal{P_{X,Y}}}\left\{f(X) \neq f^{*}(X) \wedge f^{*}(X)=Y \right\} + R(f^{*})  \nonumber\\ 
				&\leq& \Pr_{X,Y \sim \mathcal{P_{X,Y}}}\left\{f(X) \neq Y \right\} + R(f^{*}) \nonumber \\
				&=& R(f) + R(f^{*}) \nonumber\\
				&\leq& 2\cdot R(f^{*}) + 11 \cdot \frac{A}{m} + 6 \cdot \sqrt{\frac{A}{m} \cdot R(f^{*})}. \label{eq30}
			\end{eqnarray}
			It follows that
			$$
f \in B(f^{*},2\cdot R(f^{*}) + 11 \cdot \frac{A}{m} + 6 \cdot \sqrt{\frac{A}{m} \cdot R(f^{*})}) 
= B(f^{*}, R_0),
			$$
and, in particular,
			$$
			G\subseteq B(f^{*},R_{0}),
			$$
			so
			$$
			\Delta G \leq \Delta B(f^{*},R_{0}).
			$$			
			The abstain rate of $\iLESS$ equals $\Delta G$. 
			We can now use the disagreement coefficient to bound the abstain rate from above,
			\begin{eqnarray}
			\Delta G \leq  \Delta B(f^{*},R_{0}) =  \frac{\Delta B(f^{*},R_{0})}{R_{0}}\cdot R_{0} \leq \theta(R_{0}) \cdot R_{0},
			\end{eqnarray}
			which concludes the proof. 
		\end{proof}
		According to Theorem \ref{thm:LessRejection}, assuming the disagreement coefficient is $\theta(r) = O(\polylog(1/r))$ for $r \geq R(f^*)$, the rejection mass of $\iLESS$, defined as the probability that the classifier trained by $\iLESS$ will output ``I don't know'' is bounded w.h.p. by
		\begin{equation}
		\polylog_1(\frac{1}{R(f^*)+1/m}) \cdot R(f^*) + \frac{\polylog_2(m,d,1/\delta)}{m}.	\label{fast_rejection}
		\end{equation}
		In many cases, the disagreement coefficient, $\theta(r)$, is bounded by a constant, or by $O(\polylog(1/r))$ for all $r>0$ (see \cite{Hanneke_book}). For example, it was shown in \cite{JMLR:v16:wiener15a}, that for linear separators under mixture of Gaussians, and for axis-aligned rectangles under product densities over $R^k$, $\theta(r)$ is bounded by $O(\polylog(1/r))$ for all $r>0$. For such cases, we know that (\ref{fast_rejection}) always holds, regardless of the size of $R(f^*)$. The disagreement coefficient is only dependent on the marginal $\mathcal{P_{X}}$,   the hypothesis class $\mathcal{F}$, and the identity of the true risk minimizers, $f^*$ (which is 
		not necessarily unique).
		This fact motivates the following definition of a rejection rate of a selective learning algorithm, which is 
		only dependent  on $\mathcal{P_{X}}$,$\mathcal{F}$ and $f^*$.
		\begin{defi}[Fast $R^*$ Rejection Rate]
			\label{fast $R^*$ rejection rate}
			Given $\mathcal{P_{X}}$,$\mathcal{F}$ and $f^*$, if for any $\mathcal{P_{Y|X}}$, for which $f^*$ is a true risk minimizer, the rejection mass of a selective classifier learning algorithm is bounded  by probability of at least $1-\delta$ by (\ref{fast_rejection}), we say that the algorithm achieves a \textbf{fast $R^*$ rejection rate}, with $\polylog_1$ and $\polylog_2$ as its parameters.	
		\end{defi}
		Clearly, by Theorem \ref{thm:LessRejection}, if $\theta(r)=O(\polylog(1/r))$ for all $r>0$, then {\iLESS} has a fast $R^*$ rejection rate. In the next section, we will show the other direction; that is, if there is a {\PCS } learning algorithm that has a fast $R^*$ rejection rate, then $\theta(r)=O(\polylog(1/r))$ for all $r>0$.

%

		As long as the number of training 
		examples that $\iLESS$ receives is not ``too large'' relative to $1/R(f^*)$, i.e., 
		$m \ll \frac{1}{R(f^{*})}$, the rejection mass of $\iLESS$ is
		$
		O(\frac{\polylog(m,d,1/\delta)}{m}).
		$
		When $m$ is large, and $R(f^{*})$ becomes more dominant than $\frac{1}{m}$, 
		our coverage bound is dominated by $R(f^{*})$. This should not surprise us, as $\iLESS$ achieves \emph{pointwise-competitiveness} w.h.p., and any strategy that achieves pointwise-competitiveness cannot ensure a better rejection mass than $R(f^{*})$ without making more assumptions about the error or the distribution. This can be seen in the following example, in which $\theta(r)\leq 1$ for all $r>0$, but the rejection mass of any pointwise-competitive strategy is always at least $R(f^{*})$.
		\begin{example}	\label{example:1}
			Given any $0<\epsilon<0.5$, let $\cX = [0,1]$, and $\cF = \{f_1,f_2\}$ where 
			\begin{displaymath}
			f_1(x)=
			\begin{cases}
			1, & x<\epsilon \\
			0, & \text{otherwise}	
			\end{cases},	
			f_2(x)=
			\begin{cases}
			1, & x>1-\epsilon \\
			0, & \text{otherwise}.
			\end{cases}
			\end{displaymath}
		\end{example}
		Let $\mathcal{P_{X}}$ be the uniform distribution over $[0,1]$. Assume that $Y$ will always be zero.
		$f_1$ and $f_2$ are both $f^*$. Every pointwise-competitive classifier will have to output $g(x)=0$ for every $x$ in the disagreement set of $f_1$ and $f_2$. $R(f^*)=\epsilon$, and the rejection mass is $2\epsilon (= 2R(f^*))$.

		\section{From Selective Classification to the Disagreement Coefficient}
		We now turn to show a reduction from selective classification, to the disagreement coefficient.

	\begin{theorem} 
		\label{thm:PointwiseSelectiveToCoeff}
		Let $\mathcal{F}$ be a hypothesis class with a finite VC dimension $d$, and let $\mathcal{P_{X,Y}}$ be an unknown distribution. Let {\PCS} be an algorithm that returns a pointwise-competitive selective classifier w.h.p. If there exists an $m_{max}$ s.t. for every $m \leq m_{max}$, with probability of at least $1-\delta$, the abstain rate $1 - \Phi$ of  ${\PCS}(S_m,\delta, \mathcal{F}, d)$ is bounded above as follows:
		\begin{equation}
		\label{eq:probEventGeneral}
		1- \Phi(\PCS) \leq 
		\frac{\polylog(m,d,1/\delta)}{m}.
		\end{equation}
		Then for every $f^*$ (every true risk minimizer), for every $r \geq 1/m_{max}$,
		$$
		\theta_{f^*}(r) \leq 8(\polylog(1/r,d,1/r)+3).
		$$
		
	\end{theorem}
	
	\begin{proof}
		For any $m \in \left \{ 2,3,...,m_{max} \right \}$,
		denote by  $\cS_m$ a random training sample drawn from $\mathcal{P_{X,Y}}$.
		Let  $Z$ be a random variable representing a single random unlabeled example sampled from $\mathcal{P_{X}}$, and let $f^*$ to be a specific true risk minimizer. 
		
		Given $z \in DIS\left(  B(f^*,\frac{1}{m})\right)$, as used in \cite[Lemma 47]{Hanneke11}, we know that there exists a function $h_z \in \cF$ s.t. $h_z(z) \neq f^*(z)$ and $Pr(h_z(X)\neq f^*(X)) \leq \frac{1}{m}$. We denote by $\mathcal{P_{X,Y}}_z$ a new probability distribution that is identical to $\mathcal{P_{X,Y}}$ over all
		$x \in\cX$ with the exception of $\{ x : \ \ h_z(x) \neq f^*(x)\}$, over which it is defined to be $Y \eqdef h_z(x)$. It is easy to see that $h_z$ is an $f^*$ for such a distribution.  
		
		Denote by $\emph{e}_1$ the probability event where (\ref{eq:probEventGeneral}) holds (for a specific $m \leq m_{max}$).
		Denote by $\emph{e}_2$ the event where {\PCS} has succeeded in returning a pointwise-competitive selective classifier $(f_{s_m},g_{s_m})$ under $S_m$.
		
		Define $S'_m$ to be a modified $S_m$. For every $x$ s.t. $h_z(x) \neq f^*(x)$, $y$ changes to be $y=h_z(x)$. $S_m'$ is a random training sample drawn from $\mathcal{P_{X,Y}}_z$.				
		Denote by $\emph{e}_{3z}$ the event where {\PCS} has succeeded in returning a pointwise-competitive selective classifier $(f_{s'_m},g_{s'_m})$ under $S'_m$. $h_z$ is only defined for cases in which $z \in DIS\left(  B(f^*,\frac{1}{m})\right)$, and thus we define that $\emph{e}_{3z}$ will also include cases for which $z \notin DIS\left(  B(f^*,\frac{1}{m})\right)$.
		
		Under our assumptions, $\Pr(\emph{e}_1),\Pr(\emph{e}_2) \geq 1-\delta$. For every $z \in DIS (B(f^*,\frac{1}{m}))$, $\Pr(\emph{e}_{3z}|z)  \geq 1-\delta$, and for every $z \notin DIS (B(f^*,\frac{1}{m}))$, $\Pr(\emph{e}_{3z}|z)=1$, which implies that $\Pr(\emph{e}_{3z})  \geq 1-\delta$. We denote by $h_z(S_m) = f^*(S_m)$ the event where $h_z(x) = f^*(x)$ for all $x \in S_m$. The explanations for the following equations follow.

		\begin{eqnarray}
		&&\Pr\{Z \in DIS \left(B(f^*,\frac{1}{m})\right)\wedge h_z(S_m) = f^*(S_m)  \} \label{aa1}  \\
		&=& \Pr\{ Z \in DIS \left(B(f^*,\frac{1}{m})\right)   \wedge h_z(S_m) = f^*(S_m)  \wedge  \emph{e}_1 \wedge \emph{e}_2 \wedge \emph{e}_{3z}\} \label{aa2}\\  
		&& + \Pr\{ Z \in DIS \left(B(f^*,\frac{1}{m})\right)   \wedge h_z(S_m) = f^*(S_m) \ | \  \neg(\emph{e}_1 \wedge \emph{e}_2 \wedge \emph{e}_{3z}) \}\cdot \Pr(\neg(\emph{e}_1 \wedge \emph{e}_2 \wedge \emph{e}_{3z}))   \nonumber \\
		&\leq& \Pr\{ Z \in DIS \left(B(f^*,\frac{1}{m})\right)   \wedge h_z(S_m) = f^*(S_m) \wedge  \emph{e}_1 \wedge \emph{e}_2 \wedge \emph{e}_{3z}\} + 3\delta \label{aa3} \\
		&\leq& \Pr\{ g_{s_m}(Z) =0 \wedge \emph{e}_1 \wedge \emph{e}_2 \wedge \emph{e}_{3z}\} + 3\delta  \label{aa4} \\
		&\leq& \Pr\{ g_{s_m}(Z) =0 \wedge \emph{e}_1 \} + 3\delta   \\
		&\leq& \Pr\{ g_{s_m}(Z) =0 \ | \  \emph{e}_1 \} + 3\delta   \\
		&\leq&   \frac{\polylog(m,d,1/\delta)}{m} + 3\delta. \label{aa5} 
		\end{eqnarray}	
		
		In (\ref{aa1}), it is convenient to view the random experiment as if we draw $Z$ first, and then $S_m$. If $Z \in DIS (B(f^*,\frac{1}{m}))$, then consider $h_z$ to be any function that holds $h_z(Z) \neq f^*(Z)$ and $Pr(h_z(X)\neq f^*(X)) \leq \frac{1}{m}$. If $Z \notin DIS (B(f^*,\frac{1}{m}))$, then the event described in (\ref{aa1}) does not occur, and $h_z$ is undefined. 
		In (\ref{aa2}), we use conditional probability, and in (\ref{aa3}) we apply the union bound. Inequality (\ref{aa4}) is justified as follows. If $h_z(S_m) = f^*(S_m)$, then the algorithm received the same input under $\mathcal{P_{X,Y}}_z$ and $\mathcal{P_{X,Y}}$. Given that $\emph{e}_{2}$ and $\emph{e}_{3z}$ occurred, we know that the algorithm had successfully output a pointwise-competitive selective classifier for both probabilities, which means that whenever $f^*$ and $h_z$ disagree, $g_{s_m}$ has to output zero; otherwise, it will not be pointwise-competitive for one of the distributions. By the definition of $h_z$, $h_z(Z) \neq f^*(Z)$, which explains the inequality. (\ref{aa5}) is driven from the definition of $\emph{e}_1$. Taking $\delta=\frac{1}{m}$, we get,
		\begin{eqnarray}
		\Pr\{Z \in DIS \left(B(f^*,\frac{1}{m})\right)\wedge h_z(S_m) = f^*(S_m)  \} &\leq& \frac{\polylog(m,d,m)+3}{m}. \label{aa7}
		\end{eqnarray}
		The following inequalities are derived using elementary conditional probability. In Equation (\ref{aa6}) we use an argument taken from the proof of \cite[Lemma 47]{Hanneke11}. $h_z \in \left( B(f^*,\frac{1}{m})\right)$ and thus the probability that $f^*$ and $h_z$ will have the same labels over a sample of size $m$ is at least  $(1-\frac{1}{m})^m$.    
		\begin{eqnarray}
		&&\Pr\{Z \in DIS \left(B(f^*,\frac{1}{m})\right)\wedge h_z(S_m) = f^*(S_m)  \} \nonumber\\
		&=&  \Pr\{ h_z(S_m) = f^*(S_m) \ | \ Z \in DIS\left( B(f^*,\frac{1}{m})\right) \} \cdot\Pr\{Z \in DIS\left( B(f^*,\frac{1}{m})\right)\} \nonumber \\
		&\geq& (1-\frac{1}{m})^m \cdot\Pr\{Z \in DIS\left( B(f^*,\frac{1}{m})\right)\} \label{aa6} \\
		&\geq& \frac{1}{4} \cdot \Delta B(f^*,\frac{1}{m}) .\label{aa8} 
		\end{eqnarray}
		Combining (\ref{aa7}) and (\ref{aa8}), we get that for every $m \in \left \{ 2,3,...,m_{max} \right \}$,
		\begin{eqnarray}
		\frac{\Delta B(f^*,1/m)}{1/m} &\leq&   4(\polylog(m,d,m)+3).  \label{aa}
		\end{eqnarray}		
		The following inequalities follow from (\ref{aa}), and from the fact that $\Delta B(f^*,x)$ and $\polylog_1(x)$ are non-decreasing. For any $r$ in $[\frac{1}{m_{max}} ,\frac{1}{2}]$,
		\begin{eqnarray*}
			\frac{\Delta B(f^*,r)}{r} &\leq&   \frac{\Delta B(f^*,\frac{1}{\left \lfloor 1/r \right \rfloor})}{\frac{1}{\left \lfloor 1/r \right \rfloor}} \cdot \frac{1}{r \cdot \left \lfloor 1/r \right \rfloor} \\
			&\leq& 4 ( \polylog(\lfloor 1/r  \rfloor,d,\lfloor 1/r  \rfloor)+3 ) \cdot \frac{1}{r \cdot (1/r-1)} \\
			&\leq&  4 ( \polylog(\lfloor 1/r  \rfloor,d,\lfloor 1/r  \rfloor)+3 ) \cdot \frac{1}{1-r} \\
			&\leq&  8 ( \polylog(1/r ,d,1/r)+3 )
		\end{eqnarray*}		
		and for $r$ in $[\frac{1}{2},1]$,
		\begin{eqnarray}
		\frac{\Delta B(f^*,r)}{r} &\leq&   \frac{1}{1/2} = 2,
		\end{eqnarray}	
		which concludes the proof.
	\end{proof}

		\begin{corollary} 
			\label{thm:iLessToCoeff}
			Let $\mathcal{F}$ be a hypothesis class with a finite VC dimension $d$, and let $\mathcal{P_{X,Y}}$ be an unknown distribution. If there exists an $m_{max}$ s.t. for every $m \leq m_{max}$, with probability of at least $1-\delta$, the abstain rate $1 - \Phi$ of  ${\iLESS}(S_m,\delta, \mathcal{F}, d)$ is bounded above as follows:
			\begin{equation}
			\label{eq:probEvent}
			1- \Phi(\iLESS) \leq 
			\frac{\polylog(m,d,1/\delta)}{m}.
			\end{equation}
			Then for every $f^*$ (every true risk minimizer), for every $r \geq 1/m_{max}$,
			$$
			\theta_{f^*}(r) \leq 8(\polylog(1/r,d,1/r)+3).
			$$
			
		\end{corollary}

		\begin{proof}
			This is a direct result from Theorem \ref{thm:PointwiseSelectiveToCoeff}, and from the fact that {\iLESS} is PCS.
		\end{proof}	
		
		Given $\mathcal{P_{X}}$,$\mathcal{F}$ and $f^*$, if any {\PCS} has a fast $R^*$ rejection rate, we can apply Theorem \ref{thm:PointwiseSelectiveToCoeff} with a deterministic $\mathcal{P_{Y|X}}$ distribution for which $Y=f^*(X)$, and get that $R(f^*)=0$. Thus, by definition,  
		\begin{equation}
		1- \Phi(\iLESS) \leq \polylog_1(\frac{1}{R(f^*)+1/m}) \cdot 0 + \frac{\polylog_2(m,d,1/\delta)}{m}.	
		\end{equation}
		We can now apply Theorem \ref{thm:PointwiseSelectiveToCoeff} with $m_{max}=\infty$, and get that the disagreement coefficient is bounded by $\polylog(1/r)$ for all $r>0$. Thus, completing a two sided equivalence from {\PCS} with a fast $R^*$ rejection rate to a bounded disagreement coefficient for all $r>0$.
	
		\section{Active-ILESS}	\label{sec:Active-iLESS}
		
		\begin{strategy}
						
			\caption{Agnostic low-error active learning strategy (\ActiveiLESS)}
			\label{alg:Active-iLESS}
			\footnotesize {
				\begin{algorithmic}[1]
					\REQUIRE  $\epsilon \text{ and/or } m$ depending on the desired termination condition (error or labeling budget, respectively)\\
					Confidence level $\delta$\\
					Hypothesis class $\mathcal{F}$ with VC dimension $d$\\
					An unlabeled input sequence sampled i.i.d from $\mathcal{P_{X,Y}}$: $x_{1}, x_{2},x_{3},\ldots$
					\ENSURE A classifier $\hat{f}$.

				\textbf{Initialize:}

					 Set $\hat{S} = \emptyset $, $G_{0} = \cF$, $t = 1$.

				\textbf{Perform for each example $x_{t}$ received:}
					\STATE if $x_{t}\in AGR(G_{t-1})$: don't request label for $x_{t}$ and set $y_{t}=f(x_{t})$ using any $f\in G_{t-1}$
					
					otherwise: request label $y_t$.
					\STATE Set $\hat{S} = \hat{S} \cup \{(x_{t},y_{t})\}$.
					\STATE Set $\hat{f} = \hat{f}(\hat{S})$
					\STATE if $\log_2 (t) \in \mathbb{N}$: 	
					\begin{itemize}
						\item Set $\sigma_{Active}= \hat{\sigma}_{R-\hat{R}}
						\left(\frac{t}{2},\frac{\delta}{2t},d,\hat{R}(\hat{f},\hat{S})
						\right) + \bar{\sigma}_{\hat{R}-R}
						\left(\frac{t}{2},\frac{\delta}{2t},d,\hat{R}(\hat{f},\hat{S})+\hat{\sigma}_{R-\hat{R}}(\frac{t}{2},\frac{\delta}{2t},d,\hat{R}(\hat{f},\hat{S})) 
						\right)$
						\item If $\epsilon$ was given as input and $\sigma_{Active}<\epsilon$, terminate and return $\hat{f}$
						\item Set $G_{t} = \hat{\cV}\left(\hat{f},\sigma_{Active}\right)$ 
						\item Set $\hat{S} = \emptyset $.				
					\end{itemize}
					otherwise:
					\begin{itemize}
						\item $G_{t} = G_{t-1}$
					\end{itemize}
					\STATE If $m$ was given as input  and $t=m$, terminate and return $\hat{f}$
					\STATE Set t = t +1
				\end{algorithmic}
			}
		\end{strategy}		
		In this section we introduce, in Strategy \ref{alg:Active-iLESS}, an agnostic active learning algorithm called $\ActiveiLESS$. $\ActiveiLESS$ is very similar to Agnostic CAL \cite{hsu:thesis}, Algorithm 4.2 on page 36, and $A^2$ \cite{balcan2006agnostic}. Much like Agnostic CAL, $\ActiveiLESS$ creates artificial labels (step 1). The two algorithms differ mainly in that $\ActiveiLESS$ works in batches (inside each batch, the decision whether to query an example is made instantly and not at the end of the batch). This allows $\ActiveiLESS$ to be a bit more conservative with its deltas. Moreover, while Agnostic CAL requires calculation of an ERM with many constraints (defined by the function LEARN in HSU's thesis), $\ActiveiLESS$ requires a calculation of the ERM with only one constraint, as seen from the disbelief principle \cite{el2011agnostic}, already discussed in Section \ref{sec:iLESS}.
		
		Although {\iLESS} is not novel in and of itself, we use its similarity to Agnostic CAL to demonstrate a deep connection between active learning and selective classification. 
		
		In Section~\ref{sec:BatchiIless} we use $\ActiveiLESS$ to show an equivalence between 
		active learning (represented by $\ActiveiLESS$) and selective classification (represented by a variant of $\iLESS$, ``\BatchiLESS''). The introduction of these new variants facilitates a 
		straightforward proof of the equivalence relationship. 
		This equivalence implies a novel relationship between selective and active classification in the agnostic setting.
		
		We begin by analyzing $\ActiveiLESS$ and  showing that much like  $\iLESS$, 
		$f^{*} \in G_t$ in each iteration $t$.
		The low-error set $G$, maintained by  $\iLESS$, contains all the hypotheses that have an empirical error smaller than $\hat{R}(\hat{f})+ \sigma_{\iLESS}$.
		In Lemma \ref{lemma:agnostic} we showed that this condition implies that $f^{*}$
		resides within the low-error set $G$ of $\iLESS$.
		A proof that $f^{*} \in G_{t}$, after each iteration of $\ActiveiLESS$, cannot follow
		the same argument due to the fact that $\ActiveiLESS$, shown in Strategy \ref{alg:Active-iLESS}, labels by 
		itself each example whose label is not requested from the teacher,
		and obviously, since we consider an agnostic setting,  these self-labels  can differ from 
		the true labels.
		
		$\ActiveiLESS$, as seen in Strategy \ref{alg:Active-iLESS}, receives as a termination condition either $\epsilon>0$ and/or $m$, and terminates when the radius of its low-error set, $G_t$, is smaller than $\epsilon$, or when it has processed $m$ examples.
		
		$\ActiveiLESS$ changes its low-error set, $G_{t}$, only for $t$ that are natural powers of $2$. For each change, $\ActiveiLESS$ begins to create fake labels for $x_{t}\in AGR(G_{t-1})$ that may or may not be equal to the real label of $x_{t}$ (under the original distribution). In fact, this $G_{t}$ defines a new distribution, $\mathcal{P_{X,Y}}(G_{t})$, and this distribution changes for every $t$ that is a natural power of $2$. 
		With respect to a run of $\ActiveiLESS$, and $t=2^{i},i \in \mathbb{N}$, we denote by $\mathcal{P_{X,Y}}(G_{t})$, the new probability distribution implied by $G_{t}$, and the fake labels created by the algorithm. $R_{\mathcal{P_{X,Y}}(G_{t})}(f)$ will be the true risk under the new distribution, while $R_{\mathcal{P_{X,Y}}}(f)$ is the true risk of $f$ under the original distribution.
		
		\begin{defi} \label{K}
			Let $\mathcal{F}$ be a hypothesis class with a finite VC dimension $d$, and let $\mathcal{P_{X,Y}}$ be an unknown distribution. Given a run of {\ActiveiLESS}, we denote by $\cal K$ the event where both inequalities (\ref{errorbounds1}) and (\ref{errorbounds2}) hold simultaneously for every $f \in \cF$, for all iterations of {\ActiveiLESS} where $t=2^{i},i \in \mathbb{N}$. $\hat{R}(f) \eqdef \hat{R}(f,\hat{S})$ for $\hat{S}$ before it was initialized:
		\end{defi}
		\begin{equation}
		\label{errorbounds1}
		R_{\mathcal{P_{X,Y}}(G_{t})}(f) \leq \hat{R}(f) + \sigma_{R-\hat{R}}\left(\frac{t}{2},\frac{\delta}{2t},d,R(f),\hat{R}(f)\right)
		\end{equation}
		\begin{equation}
		\label{errorbounds2}
		\hat{R}(f) \leq R_{\mathcal{P_{X,Y}}(G_{t})}(f) + \sigma_{\hat{R}-R}\left(\frac{t}{2},\frac{\delta}{2t},d,R(f),\hat{R}(f)\right)
		\end{equation}	
		
		\begin{lemma}
			\label{lemma:errorbounds}
			$\mathcal{K}$ occurs with probability of at least $1-\delta$.
		\end{lemma}
		\begin{proof}
			$G_{t}$ changes only for iterations of the type $2^i, i \in \mathbb{N}$.  We know by Lemma \ref{lemma:agnostic} that the probability that
			inequalities (\ref{errorbounds1}) and (\ref{errorbounds2})
			do not hold is smaller than $\delta/(2t)$. By the union bound, the probability that one of these inequalities does not hold after any iteration is smaller than
			\begin{displaymath}
			\sum_{t=2^{i},i \in \mathbb{N}}\frac{\delta}{2t}\leq \delta.
			\end{displaymath}
		\end{proof}

		\begin{lemma}
			\label{lemma:f_star_best}
			If $f^{*}$, a true risk minimizer under probability distribution $\mathcal{P_{X,Y}}$, resides within $G_{t}$, then it is also a true risk minimizer under probability distribution $\mathcal{P_{X,Y}}(G_{t})$.
		\end{lemma}
		
		\begin{proof}
			\begin{displaymath}
			\argmin_{f \in \cF} R_{\mathcal{P_{X,Y}}(G_{t})}(f)=\argmin_{f \in \cF} \left( \underbrace{R_{\mathcal{P_{X,Y}}}(f)}_{A}+\underbrace{R_{\mathcal{P_{X,Y}}(G_{t})}(f)-R_{\mathcal{P_{X,Y}}}(f)}_{B} \right).
			\end{displaymath}
			We know that $f^{*}$ minimizes $A$, and we note that every function that resides within $G_{t}$ minimizes $B$, because every difference in the labeling between $\mathcal{P_{X,Y}}$ and $\mathcal{P_{X,Y}}(G_{t})$ was done according to the label given by the unanimous decision of functions in $G_{t}$. Hence, $f^{*}$ minimizes $A+B$. 
		\end{proof}	
		The proofs of the following four lemmas appear in the appendix. They all show basic good qualities of {\ActiveiLESS}. 		
		\begin{lemma}
			\label{lemma:f*in}
			Given that event $\cal K$ (see Definition \ref{K}) occurred, each $f^{*}$ of the \textbf{original} distribution $\mathcal{P_{X,Y}}$ resides within $G_{t}$ for all $t$.
			This implies that $R_{\mathcal{P_{X,Y}}(G_{t})}(f^{*}) \leq R(f^*)$, for all $t$, as every change in the labeling is done according to $f^*$. 
		\end{lemma}

		\begin{lemma}
			\label{lemma:epsilon}
			Given that event $\cal K$ (see Definition \ref{K}) occurred, and under the assumption that $\ActiveiLESS$ terminated with the $\epsilon$ condition, the hypothesis returned by $\ActiveiLESS$, $\hat{f}$, holds: 
			\begin{displaymath}
			R_{\mathcal{P_{X,Y}}}(\hat{f}) \leq R_{\mathcal{P_{X,Y}}}(f^{*}) + \epsilon.
			\end{displaymath}
		\end{lemma}

		\begin{lemma}
			\label{lemma:radius_active}
			Given that event $\cal K$ (see Definition \ref{K}) occurred, the final radius of $\ActiveiLESS$ satisfies
			\begin{eqnarray}
			\sigma_{Active}=O(\frac{B}{m}+ \sqrt{\frac{B}{m} \cdot R(f^{*})} ),
			\end{eqnarray}	
			where
			\begin{math}
				B \eqdef 16d \ln( \frac{16m^2e}{d\delta} ).
			\end{math}		 
		\end{lemma}

		\begin{lemma}
			\label{lemma:max_examples_observed}
			Given that event $\cal K$ (see Definition \ref{K}) occurred, the total number of examples that ${\ActiveiLESS(\epsilon)}$ processed (without necessarily requesting labels) is 
			$$
			O \left( \frac{1}{\epsilon}\ln(\frac{1}{\epsilon})+ \frac{R(f^*)}{\epsilon^2}\ln(\frac{R(f^*)}{\epsilon^2}  ) \right),
			$$ 	 
			where we hide factors of $d, \ln(1/\delta)$ under the $O$.
		\end{lemma}

		\begin{defi} \label{passive example complexity}
			An active learner that
			generates a hypothesis 
			whose true error is smaller than $\epsilon$ w.h.p.,
			has \textbf{passive example complexity}, if it observes up to
			$O
			\left( \frac{1}{\epsilon}\ln(\frac{1}{\epsilon})+ \frac{R(f^*)}{\epsilon^2}\ln(\frac{R(f^*)}{\epsilon^2}  ) \right)$
			examples (not necessarily labeled).		
		\end{defi}

		By Lemmas \ref{lemma:epsilon} and \ref{lemma:max_examples_observed} we know that
		{\ActiveiLESS} has passive example complexity. 
		
		The definition of a fast $R^*$ rejection rate for selective classification
		induces the following related definition for exponential speedup of active learning algorithms.
		
		\begin{defi}[$R^*$ Exponential Speedup]	
			\label{$R^*$ exponential speedup}
			Given $\mathcal{P_{X}}$,$\mathcal{F}$ and $f^*$, we say that an active learner has \textbf{$R^*$ exponential speedup}, with $\polylog_1$ and $\polylog_2$ as its parameters, if for every $\mathcal{P_{Y|X}}$ for which $f^*$ is a true risk minimizer, and for every $m>0$, with probability of at least $1-\delta$, the number of labels requested by the active learner after observing $m$ examples is not greater than 
			$$
			\polylog_1(\frac{1}{R(f^*)+1/m}) \cdot R(f^*)m+ \polylog_2(m,d,1/\delta).
			$$  
		\end{defi}

		In \cite{hsu:thesis}, Hsu  introduced  the agnostic CAL algorithm and showed (Theorem 4.3, page 41) that if the disagreement coefficient is bounded, then Agnostic CAL has $R^*$ exponential speedup (under our new definition).
		Any active algorithm that has
		passive example complexity  and achieves $R^*$ exponential speedup requires w.h.p. no more than $O\left(\polylog(\frac{R(f^*)}{\epsilon^2}) \frac{R(f^*)^2}{\epsilon^2} + \polylog(\frac{1}{\epsilon}) \right)$ labels  to reach a true error smaller than $\epsilon$. The proof is immediate by considering the cases $\frac{R(f^*)}{\epsilon} \geq 1$ and $\frac{R(f^*)}{\epsilon} < 1$. The leading term of this bound
		is $\frac{R(f^*)^2}{\epsilon^2}$, which is also the case for $A^2$ \cite{balcan2006agnostic}.

		\section{A Reduction from Active-iLess to Batch-ILESS}
		\label{sec:BatchiIless}


		In Strategy \ref{alg:Batch iLESS} we define a selective classifier, called $\BatchiLESS$, which uses 
		$\ActiveiLESS$ as its engine. Given a labeled sample $S_m$, $\BatchiLESS$ simulates the active algorithm,
		by applying it over a uniformly random ordering of $S_m$ in a straightforward manner (i.e., it 
		sequentially introduces to the active
		algorithm an unlabeled example and reveals the label only if the active algorithm requests it).
		Upon termination, after the active algorithm has consumed all examples, our batch algorithm receives $\hat{f}$ from the active algorithm and utilizes its last low-error set $G_t$ to define its selection function.

		Lemma \ref{lemma:f*in} implies that $\BatchiLESS$ is pointwise-competitive. We note that Lemma \ref{lemma:radius}, Theorem \ref{thm:LessRejection} and Theorem \ref{thm:iLessToCoeff}, 
		which were proven for \iLESS, can also be proven for \BatchiLESS. We chose to prove it for \iLESS, as it is more simple than \BatchiLESS, and doesn't require an active algorithm as its engine. We state these ideas formally, and give sketches for their proofs, in the Appendix in Lemma \ref{lemma:radius_batch} and Theorem \ref{thm:LessRejection_batch}.
		 

				\begin{strategy}[]
					\caption{Batch Improved Low-Error Selective Strategy (\BatchiLESS)}
					\footnotesize {
						\begin{algorithmic}[1]
							\REQUIRE  Sample set of size $m$, $S_m$,\\
							Confidence level $\delta$\\
							Hypothesis class $\mathcal{F}$ with VC dimension $d$\\
							\ENSURE A selective classifier $(h,g)$ 
							\STATE Simulate $\ActiveiLESS$ with a random ordering of $S_m$ as its input stream; 
							let $G_t$ be the low-error set obtained by $\ActiveiLESS$ in its last round, and let $\hat{f}$
							be its resulting classifier.
							\STATE Construct $g$ such that $g(x) = 1 \Longleftrightarrow x \in \left\{\cX \setminus DIS\left(G_t\right)\right\}$
							\STATE $h = \hat{f}$
						\end{algorithmic}
					}
					\label{alg:Batch iLESS}
				\end{strategy}
		
		The following theorem shows a deep connection between the speedup of {\ActiveiLESS} to the rejection mass of {\BatchiLESS} for specific $\mathcal{P_{X,Y}}$. An immediate 
		corollary of this theorem is that if {\ActiveiLESS} has $R^*$ exponential speedup (see Definition \ref{$R^*$ exponential speedup}), then {\BatchiLESS} has a fast $R^*$ rejection rate
		(see Definition \ref{fast $R^*$ rejection rate}).   	
		\begin{theorem} 
			\label{thm:activeToSelective}
			Let $\mathcal{F}$ be a hypothesis class with a finite VC dimension $d$, and let $\mathcal{P_{X,Y}}$ be an unknown distribution. If after observing $m$ examples, with probability of at least $1-\delta$, the number of labels requested by $\ActiveiLESS$ is not greater than 
			\begin{equation}
			\polylog_1(\frac{1}{R(f^*)+1/m}) \cdot R(f^*)m+ \polylog_2(m,d,1/\delta), \label{eq:active_speedup}
			\end{equation}
			then the rejection mass of {\BatchiLESS} is bounded w.h.p. by	
			$$
			8 \cdot \polylog_1(\frac{1}{R(f^*)+1/m}) \cdot R(f^*) + \frac{2\left( \sqrt{\ln(2/\delta)} + \sqrt{\ln(2/\delta)  +2\polylog_2(2m,2/\delta)}\right)^2 }{m}	.
			$$
		\end{theorem}
		\begin{proof}
			Consider an application  of $\ActiveiLESS$ 
			with $\delta = \delta_0$
			over $m_0 \eqdef 2^{\left \lceil \log(m+1) \right \rceil}$ examples. 
			Denote by $X_i$ an indicator random variable for the labeling of its $i$th example,
			$1 \leq i \leq m_0$. 
			With probability of at least $1-\delta_0$ over the choice of samples from $\mathcal{P_{X,Y}}$, 
			\begin{eqnarray}
			\sum_{i=1}^{m_0} X_i \leq \polylog_1(\frac{1}{R(f^*)+1/m_0}) \cdot R(f^*)m_0+ \polylog_2(m_0,1/\delta_0). \label{h1}
			\end{eqnarray}		
			We know by the definition of $\ActiveiLESS$ (Strategy \ref{alg:Active-iLESS}), that the last $m_0/2$ examples had the exact same probability, $\Delta G_{m_{0}/2}$, of requiring a label, and that this is exactly the probability that $\BatchiLESS$ will decide to abstain after receiving $m$ examples, according to Strategy \ref{alg:Batch iLESS}. 
			
			We now estimate $\Delta G_{m_{0}/2}$
			using the following version of the Chernoff bound given by Canny \cite{JohnCannyNotes}. 
			For the sake of self-containment, Canny's statement and proof of the bound are provided in Lemma \ref{ChernoffBound} in the Appendix.
		     
		    The statement of the lemma is as follows. 
		     Let $X_1,X_2,\ldots,X_n$ be independent Bernoulli trials with $Pr[X_i=1]=p$, let $X\eqdef\sum_{i=1}^{n} X_i$, and $\mu = \E X$. Then, for every $\alpha>0$: $$\Pr \left( X<(1-\alpha)\mu \right)   \leq \exp(-\mu \alpha^2 /2).$$
			
			Applying the Chernoff bound with the indicator variables of the last $m_0/2$ examples, we have $X=\sum_{m_0/2}^{m_0} X_i$, $\mu = p\frac{m_0}{2}$, and set 
			$p \eqdef \Delta G_{m_{0}/2}$.
			Select $\alpha$ such that 
			$$
			\exp(- p\frac{m_{0}}{2} \alpha^2 /2) = \delta_2.
			$$ 
			Solving for $\alpha$,
			$$
			\alpha =\sqrt{ \frac{4 \ln(1/\delta_1)}{m_0p}}. 
			$$
			We conclude that with probability of at least $1-\delta_1$, 
			\begin{eqnarray}
			& & X \geq (1-\sqrt{ \frac{4 \ln(1/\delta_1)}{m_0p}}) \cdot p\frac{m_0}{2} \nonumber \\
			&\Leftrightarrow& 0 \geq \frac{pm_0}{2} - \sqrt{pm_0 \cdot \ln(1/\delta_1)} -X \label{eq:quadratic}.
			\end{eqnarray}	
			Solving the quadratic equation~(\ref{eq:quadratic}) for $\sqrt{pm_0}$, we get that 
			\begin{eqnarray}
			&&\sqrt{pm_0} \leq  \frac{\sqrt{\ln(1/\delta_1)} + \sqrt{\ln(1/\delta_1) + 2X} }{1} \nonumber \\
			& \Rightarrow & p \leq \frac{(\sqrt{\ln(1/\delta_1)} + \sqrt{\ln(1/\delta_1) + 2X})^2 }{m_0} \label{h2} .
			\end{eqnarray}
			Combining (\ref{h1}) and (\ref{h2}), from the union bound we get that with probability of at least $1-\delta_0 -\delta_1$,
			\begin{eqnarray*}
			\Delta G_{m_{0}/2} \leq \frac{\left( \sqrt{\ln(1/\delta_1)} + \sqrt{\ln(1/\delta_1) + 2\polylog_1(\frac{1}{R(f^*)+1/m_0}) \cdot R(f^*)m_0+2\polylog_2(m_0,1/\delta_0)} \right)^2 }{m_0}.
			\end{eqnarray*}
			If we take $\delta_0=\delta_1=\delta/2$, then, since $m\leq m_0 \leq 2m$, we can use $\sqrt{a+b} \leq \sqrt{a} + \sqrt{b}$
			and $(a+b)^2 \leq 2a^2 + 2b^2$,
			to obtain
			\begin{eqnarray*}
			\Delta G_{m_{0}/2} &\leq& \frac{\left( \sqrt{\ln(2/\delta)} + \sqrt{\ln(2/\delta) + 4\polylog_1(\frac{1}{R(f^*)+1/m}) \cdot R(f^*)m+2\polylog_2(2m,2/\delta)} \right)^2 }{m} \\
			&\leq& \frac{\left( \sqrt{\ln(2/\delta)} + \sqrt{\ln(2/\delta) +2\polylog_2(2m,2/\delta)} +\sqrt{4\polylog_1(\frac{1}{R(f^*)+1/m}) \cdot R(f^*)m} \right)^2 }{m}\\
			&\leq& \frac{2\left( \sqrt{\ln(2/\delta)} + \sqrt{\ln(2/\delta) +2\polylog_2(2m,2/\delta)}\right)^2 + 2\left( \sqrt{4\polylog_1(\frac{1}{R(f^*)+1/m}) \cdot R(f^*)m} \right)^2 }{m}\\
			&=& \frac{2\left( \sqrt{\ln(2/\delta)} + \sqrt{\ln(2/\delta)  +2\polylog_2(2m,2/\delta)}\right)^2 }{m}	+ 8\cdot \polylog_1(\frac{1}{R(f^*)+1/m}) \cdot R(f^*)
			\end{eqnarray*}

		\end{proof}
		
		\begin{corollary} 
			\label{thm:ActiveiLessToCoeff}

			Let $\mathcal{F}$ be a hypothesis class with a finite VC dimension $d$, and let $\mathcal{P_{X,Y}}$ be an unknown distribution. If after observing $m$ examples, with probability of at least $1-\delta$, the number of labels requested by $\ActiveiLESS$ is not greater than 
			\begin{equation*}
				\polylog_1(\frac{1}{R(f^*)+1/m}) \cdot R(f^*)m+ \polylog_2(m,d,1/\delta), 
			\end{equation*}
			then for every $r \geq R(f^*)$,
			$$
			\theta_{f^*}(r) \leq 8 \left( 2\left( \sqrt{\ln(2r)} + \sqrt{\ln(2r)  +2\polylog_2(2/r,2/r)}\right)^2 +8\cdot \polylog_1(1/r)+2 \right) = O(\polylog(1/r)).
			$$
		\end{corollary}
		
		\begin{proof}
			The proof follows from Theorems \ref{thm:activeToSelective} and \ref{thm:PointwiseSelectiveToCoeff}. 
			Applying Theorem \ref{thm:activeToSelective}, we know that for $m \leq 1/R(f^*)$, the rejection mass of {\BatchiLESS} is bounded w.h.p. by,
			$$
			\frac{2\left( \sqrt{\ln(2/\delta)} + \sqrt{\ln(2/\delta)  +2\polylog_2(2m,2/\delta)}\right)^2 +8\cdot \polylog_1(\frac{1}{R(f^*)+1/m})}{m}.
			$$ 
			Applying Theorem \ref{thm:PointwiseSelectiveToCoeff} with $m_{max} = 1/R(f^*)$, we get that for every $r \geq R(f^*)$,
			\begin{eqnarray*}
			\theta_{f^*}(r) &\leq& 8 \left(2\left( \sqrt{\ln(2r)} + \sqrt{\ln(2r)  +2\polylog_2(2/r,2/r)}\right)^2 +8 \cdot \polylog_1(\frac{1}{R(f^*)+r})+3 \right)\\
			&\leq&  8 \left( 2\left( \sqrt{\ln(2r)} + \sqrt{\ln(2r)  +2\polylog_2(2/r,2/r)}\right)^2 +8\cdot \polylog_1(1/r)+3 \right).
			\end{eqnarray*}
			Note that the Theorem \ref{thm:PointwiseSelectiveToCoeff} does not require $m_{max}$ to be an integer.
		\end{proof}

		\section{From the Disagreement Coefficient to Active Learning}

		In this section we show that when $\theta'(r)$ is bounded by $\polylog_1(1/r)$ for all $r>R(f^*)$	for some specific $\mathcal{P_{X,Y}}$, then the label complexity of {\ActiveiLESS} under the same $\mathcal{P_{X,Y}}$ is bounded by 
		\begin{equation}
		\polylog_2(\frac{1}{R(f^*)+1/m}) \cdot R(f^*)m+ \polylog_3(m,d,1/\delta), \label{eq:active_speedup4}
		\end{equation}
		where the parameters of $\polylog_2$ and $\polylog_3$ are only dependent on $\polylog_1(1/r)$.
		Thus, if $\theta'(r)\leq \polylog_1(1/r)$ for all $r>0$, we get that {\ActiveiLESS} has $R^*$ exponential speedup.		
		This direction has been shown before in \cite{hsu:thesis,Hanneke07} for agnostic CAL and $A^2$. For the sake of
		self-containment, we show it here for {\ActiveiLESS}. Due to the fact that {\ActiveiLESS} relies on {\iLESS}, which we already have bounds for, the proof is straightforward.
		
		As a preparation for the theorem, we present Lemma \ref{lemma:disagreement_monotonicity} (shown before in \cite{Hanneke_book}), in which we introduce a small feature of the disagreement coefficient that will serve us later. 
		\begin{lemma}
			\label{lemma:disagreement_monotonicity}
			Let $\mathcal{F}$ be a hypothesis class with a finite VC dimension $d$, and let $\mathcal{P_{X,Y}}$ be an unknown distribution. For every $f \in \mathcal{F}$ and $0<r\leq 1$, $\theta_{f}(r) \cdot r$ is a non-decreasing function. 	 
		\end{lemma}
		\begin{proof}
			Given $0<r_1<r_2$, we will show that $\theta_{f}(r_1) \cdot r_1 \leq \theta_{f}(r_2) \cdot r_2$. Assume by contradiction that 
			
			$$
			\theta_{f}(r_1) \cdot r_1 > \theta_{f}(r_2) \cdot r_2,
			$$
			i.e., 
			$$
			\sup_{r>r_1}\frac{\Delta B(f,r_1)}{r_1}  \cdot r_1 > \sup_{r>r_2}\frac{\Delta B(f,r_2)}{r_2} \cdot r_2.
			$$	
			This implies, that there exists $r_1\leq \hat{r} < r_2$ s.t.
			$$
			\frac{\Delta B(f,\hat{r})}{\hat{r}}   r_1 > \sup_{r>r_2}\frac{\Delta B(f,r_2)}{r_2}  r_2 \geq \frac{\Delta B(f,r_2)}{r_2}  r_2 = \Delta B(f,r_2).
			$$	
			This contradicts the known monotonicity of $\Delta B(f,x)$.	
		\end{proof}
		
		\begin{theorem} 
			\label{thm:coeffToActiveLearning}
			Let $\mathcal{F}$ be a hypothesis class with a finite VC dimension $d$, let $\mathcal{P_{X,Y}}$ be an unknown distribution, and $f^*$ is a true risk minimizer of $\mathcal{P_{X,Y}}$. If for all $r>R(f^*)$,	
			$$
			\theta'(r)\leq \polylog_1(1/r),
			$$ 
			then the label complexity of $\ActiveiLESS (m,\delta/2)$ is bounded by
			$$
			\polylog_1 \left(\frac{1}{5 R(f^{*}) + 14 \frac{A}{m}} \right) 2e \cdot mR(f^{*}) + 	\log_2(2/\delta) + 56 e \cdot  \log_2 m \cdot A \cdot \polylog_1
			\left(\frac{1}{5 R(f^{*}) + 14  \frac{A}{m}} \right) ,
			$$
			which has the same form of Equation (\ref{eq:active_speedup4}).
			
		\end{theorem}
		
		\begin{proof}	
			Each run of $\ActiveiLESS(m,\delta/2)$ simulates $\log_2 m$ runs of $\iLESS$. We know by Lemma \ref{lemma:errorbounds} that with probability of at least $1-\delta/2$, inequalities (\ref{errorbounds1}) and (\ref{errorbounds2}) hold for each run. Recall that we denoted by $\cal K$ the event where both inequalities hold through out all runs of  $\iLESS$, which is exactly the definition of event $\cal E$ per run (see Definition \ref{E}). Under event $\cal K$, Lemma \ref{lemma:f*in} implies that all $f^{*}$ of the original distribution $\mathcal{P_{X,Y}}$ reside within $G_{t}$ for all $t$. This also implies that all $f^*$ of the original distribution remain the true risk minimizers under $\mathcal{P_{X,Y}}(G_{t})$, for all $t$, as they always benefit from the creation of the artificial labels.   
			
			Because the marginal of the distribution does not change during the run of {\ActiveiLESS}, and because event $\cal E$ holds for each iteration of {\iLESS}, we can apply Theorem \ref{thm:LessRejection} for all of the runs of {\iLESS}. We thus get that for every run of {\iLESS}, the rejection mass is bounded by
			$$
			1-\Phi(\iLESS) \leq \theta(R_{0}) \cdot R_{0},	
			$$	
			where
			$$
			R_{0} \eqdef 2\cdot R(f^{*}) + 11 \cdot \frac{A}{m} + 6 \cdot \sqrt{\frac{A}{m} \cdot R(f^{*})}.
			$$
			We denoted by $R(f^{*})$ the true error according to the original distribution, which might be larger than the true error implied by the fake label distributions that the algorithm induces. However, according to Lemma \ref{lemma:disagreement_monotonicity}, enlarging $R_{0}$ can only weaken the bound, and thus, there is no problem doing so. We additionally bound $R_{0}$ using $\sqrt{AB} \leq A/2 + B/2$ to get
			$$
			R_{0} \leq 5\cdot R(f^{*}) + 14 \cdot \frac{A}{m}.
			$$
			Given our bound on the disagreement coefficient, we conclude that
			$$
			1-\Phi(\iLESS) \leq  \polylog_1(\frac{1}{5\cdot R(f^{*}) + 14 \cdot \frac{A}{m}}) \cdot (5\cdot R(f^{*}) + 14 \cdot \frac{A}{m}).	
			$$	
			Each activation of $\iLESS$ has delta equals $\frac{\delta}{4t}$, and thus, exactly as in Lemma \ref{lemma:errorbounds}, with probability of at least $1-\delta/2$, they all have a bounded rejection mass simultaneously. We assume that this event occurred. 
			According to the definition of $G_t$ in Strategy \ref{alg:Active-iLESS}, the probability distribution of the artificial labeling done by {\ActiveiLESS} changes only when $t$  is   a natural power of 2. Thus, the probability of requesting label $t>2$, denoted by $P_t$, is bounded by
			
			\begin{equation}	\label{ccc}
			P_t \leq \polylog_1
			\left(\frac{1}{5\cdot R(f^{*}) + 14 \cdot \frac{A}{T}} \right) \cdot 5R(f^*) + \frac{14A \cdot \polylog_1
				\left(\frac{1}{5\cdot R(f^{*}) + 14 \cdot \frac{A}{T}} \right)}{T},  
			\end{equation}
			where $T=2^{\left \lfloor \log_2 (t-1) \right \rfloor -1}$.
			
			We now have a series of Poisson trials, $X_1,X_2,\ldots,X_m$, with $Pr(X_t = 1)=P_t$, and each $X_i$ is an indicator variable for the labeling of the $i$th example. We use a version of the Chernoff bound \cite{JohnCannyNotes} to bound the label complexity.\footnote{We found this useful bound in \cite{hannekestatistical} (Theorem 5.4).}
			 The statement and a sketch of the proof of this bound are provided in Lemma \ref{ChernoffBound2} in the Appendix.
			 
			For independent Poisson variables $X_1,X_2,\ldots,X_m$, where $Pr[X_i=1]=p_i$,  $X\eqdef\sum_{i=1}^{n} X_i$, and  $\mu = \E X$, for every $\alpha>2 e -1$: 
			$$
			\Pr (X>(1+\alpha)\mu) \leq 2^{-\mu \alpha}.
			$$
			To bound $\mu = \E X$ from above, we use inequality (\ref{ccc}) and plug it into the definition of $\mu$.   
			\begin{eqnarray}
			\mu &=& P_1 + P_2 + \sum_{i=3}^{m} P_t \nonumber \\
			&\leq& 2 + \sum_{k=1}^{\log_2 m -1} 2^{k}P_{2^{k+1}} \nonumber \\
			&\leq& 2 + m \cdot \polylog_1 \left(\frac{1}{5 R(f^{*}) + 14  \frac{A}{m}} \right) \cdot R(f^*) + \sum_{k=1}^{\log_2 m -1} 2^{k} \frac{14A \cdot \polylog_1
				\left(\frac{1}{5 R(f^{*}) + 14  \frac{A}{m}} \right)}{2^{k-1}} \nonumber \\
			&\leq& 2 + m \cdot \polylog_1 \left(\frac{1}{5 R(f^{*}) + 14  \frac{A}{m}} \right) \cdot R(f^*) + 28 \log_2 m \cdot A \cdot \polylog_1
			\left(\frac{1}{5 R(f^{*}) + 14  \frac{A}{m}} \right) . \nonumber \\
			\label{fg}
			\end{eqnarray}
			We need to choose an $\alpha$ that satisfies both  
			$ 2^{-\mu \alpha} \leq \delta / 2, $ and $ \alpha > 2e-1$. 
			Clearly, $\alpha = \frac{\log_2(2/\delta)}{\mu}+2e-1$ suffices. Hence, we get that with probability of at least $1-\delta/2$,
			\begin{eqnarray*}
				X &\leq& (1+\frac{\log_2(2/\delta)}{\mu}+2e-1)\mu \\
				&=& \log_2(2/\delta)+2e\mu.
			\end{eqnarray*}
			Inequality (\ref{fg}) holds with probability of at least $1-\delta/2$, and using the union bound, we get that with probability of at least $1-\delta$,
			\begin{eqnarray}
				X &\leq& \log_2(2/\delta)+2e\left( 2 + m \cdot \polylog_1 \left(\frac{1}{5 R(f^{*}) + 14 \frac{A}{m}} \right) \cdot R(f^*) + 28 \log_2 m \cdot A \cdot \polylog_1
				\left(\frac{1}{5 R(f^{*}) + 14  \frac{A}{m}} \right) \right) \nonumber \\
				&=&   \polylog_1 \left(\frac{1}{5 R(f^{*}) + 14 \frac{A}{m}} \right) 2e \cdot mR(f^{*}) + 	\log_2(2/\delta) + 56 e \cdot  \log_2 m \cdot A \cdot \polylog_1
				\left(\frac{1}{5 R(f^{*}) + 14  \frac{A}{m}} \right) 			 \nonumber \\ \label{bbb}
			\end{eqnarray}
		\end{proof}
		The dominant factor of Equation (\ref{bbb}), if we ignore the logarithmic factors, is $m R(f^*)$. {\ActiveiLESS} has passive example complexity (see Definition \ref{passive example complexity}), which means that the total sample complexity is bounded by $\tilde{O}( \frac{1}{\epsilon}+ \frac{R(f^*)}{\epsilon^2} )$,  
		where $\tilde{O}( \cdot)$ hides logarithmic factors.
		Plugging the sample complexity into $m$ in (\ref{bbb}), we get that the total label complexity is bounded by  $\tilde{O}(\frac{R(f^*)^2}{\epsilon^2})$, in cases for which $\iLESS$ has a fast $R^*$ rejection rate. In \cite[Theorem 3]{kar2006}, Kääriäinen showed  that for every active learning algorithm, under a specific (non-trivial) hypothesis class $\cal F$, there exists a deterministic target function $g$, and a marginal distribution $\mathcal{P_{X}}$, s.t. the label complexity is $\tilde{\Omega}(\frac{R(f^*)^2}{\epsilon^2})$ (where
		$\tilde{\Omega}(\cdot)$ hides logarithmic factors). 
	
		\section{Concluding Remarks}
			
		In this paper we focused on disagreement-based methods. Namely, we always required that $f^*$  remain inside a low-error subset of hypotheses 
		w.h.p., and made decisions based on disagreement considerations. 
		We introduced a new selective classification algorithm, called $\iLESS$, whose rejection ``engine''
		utilizes sharp generalization bounds (which depend on $R(f^*)$). Our analysis proves that $\iLESS$ 
		has sometimes significantly better rejection guarantees relative to the best known
		pointwise-competitive selective strategy of \cite{WienerE14}.
		Moreover, the guarantees we provide for $\iLESS$ do not depend at all on the Bernstein assumption.
		For the general agnostic setting, we showed an equivalence relation between pointwise-competitive selective classification, active learning, and the disagreement coefficient (see Figure~\ref{figure1}).
		This equivalence is formulated in terms of a fast $R^*$ rejection rate 
		and $R^*$ exponential speedup (Definitions~\ref{fast $R^*$ rejection rate} and~\ref{$R^*$ exponential speedup}). 
		
		Theorems \ref{thm:PointwiseSelectiveToCoeff} and \ref{thm:LessRejection} show that selective classification with a fast $R^*$ rejection rate is completely equivalent to having a disagreement coefficient bounded by $\polylog(1/r)$ for $r>0$. 
		In Section \ref{sec:Active-iLESS}, in Strategy \ref{alg:Active-iLESS}, we define {\ActiveiLESS} using {\iLESS} implicitly as its engine (see State 4 in Strategy \ref{alg:Active-iLESS}). We can replace {\iLESS} with another pointwise-competitive selective algorithm, and thus construct a new active learner, that queries a label whenever the selective classifier abstains, and create a fake label according to the decision of the classifier whenever it decides to predict. Because the selective predictor is pointwise-competitive, we know that
		the underlying distribution induced by its
		fake labels is equivalent to a distribution defined by a  deterministic labeling according to $f^*$ and the same $\mathcal{P_{X}}$. 
		The algorithm will terminate using the exact same termination condition as {\ActiveiLESS} (when $\sigma_{Active}<\epsilon$), and thus the total sample complexity (labeled and unlabeled examples) will remain the same. The change will only be in the labeling criterion. Lemmas	\ref{lemma:f_star_best}, \ref{lemma:f*in}, \ref{lemma:epsilon}, \ref{lemma:radius_active}, and \ref{lemma:max_examples_observed} can all be generalized to such an algorithm.

		Going in the other direction to create a selective classifier from a general active learner is more challenging. However, if the active learner follows the {\ActiveiLESS} paradigm, and in particular,
		uses a pointwise-competitive selective classifier to decide on label requests, then a new pointwise-competitive selective classifier can be created in the same way that {\BatchiLESS} was created, and then we can obtain a restatement of Theorem \ref{thm:activeToSelective} 
		providing a reduction from an $R^*$ exponential speedup of the active algorithm, to a fast $R^*$ rejection rate of the selective classifier.

		Disagreement-based decision making in active and selective learning leads to  ``defensive''
		algorithms. For example, in the active learning case, this means that a defensive 
		algorithm will ask for more labels than a more aggressive algorithm.
		In selective classification, this defensiveness provides the power to be pointwise-competitive, but will entail an increased rejection rate.
		It would be interesting to consider more aggressive algorithms that could, for example, take into consideration an estimation of $\mathcal{P_{X}}$ in order to ignore examples 
		that cause disagreement only between functions that are very similar to each other (in terms of the
		probability mass of their difference). Such algorithms can be seen in \cite{Dasgupta05coarse, FreundEtAl97a,gonen2013efficient}, for the realizable and the low error scenarios. We believe that there is still work to be done for the agnostic scenario.

		Many aggressive algorithms could be devised under assumptions 
		about knowledge of  $\mathcal{P_{X}}$ (that could be acquired during the run of the algorithm, and 
		is given in the transductive case), or in a Bayesian setting where a prior distribution on $\cF$ exists.
		When researching this direction, one might also want to define a cost over unlabeled examples, and discuss the trade-off between labeled and unlabeled examples. 
		The main open question inspired by our results would be to identify similar correspondence between 
		aggressive selective classification algorithms and aggressive active learners.  
 
		Another aspect of selective classification and active learning, which was not addressed in this paper, is differentiating between more and less noisy areas of the distribution. A noisy area could be defined as an area for which even the best classifier in the class could not achieve a low-error. This motivates a new type of labeling for selective prediction, where one can abstain for two reasons:
		(i) lack of knowledge in a specific region of $\cX$, i.e., not enough examples were observed in that region, and the generalization bounds are not sufficiently tight.
		(ii) The region was well explored, but even the best classifier performs poorly, and thus the answer is unknown (the region is noisy). 
		In our paper, an active learner will query for both scenarios; however, a more clever active learner might only query examples of the first type, as examples of the second type cannot
		reduce its error.

		\section*{Acknowledgments}
		This research was supported by The Israel Science Foundation (grant No. 1890/14)

		\appendix
		\section{}

		\begin{proof} [of Lemma \ref{lemma:f*in}]
			We prove the claim by induction over $t$ for which $G_{t}$ is different from $G_{t-1}$. The base case of the induction is clear.
			We now show that functions that are true risk minimizers of $\mathcal{P_{X,Y}}(G_{t-1})$ reside within $G_{t}$. According to Lemma~\ref{lemma:f_star_best}, $f^{*}$ is a true risk minimizer under $\mathcal{P_{X,Y}}(G_{t-1})$ (given the induction hypothesis), and hence will also be within $G_{t}$. 
			We refer by $f^{*}$ to a true risk minimizer according to $\mathcal{P_{X,Y}}(G_{t-1})$. Using inequality (\ref{errorbounds2}) and the definition of $\bar{\sigma}_{\hat{R}-R}$,

			\begin{eqnarray}
			\hat{R}(f^{*},\hat{S}) &\leq& R_{\mathcal{P_{X,Y}}(G_{t-1})}(f^{*})+\sigma_{\hat{R}-R}\left(\frac{t}{2},\frac{\delta}{2t},d,R_{\mathcal{P_{X,Y}}(G_{t-1})}(f^{*}),\hat{R}(f^{*},\hat{S})\right) \nonumber \\
			&\leq& R_{\mathcal{P_{X,Y}}(G_{t-1})}(f^{*})+\bar{\sigma}_{\hat{R}-R}\left(\frac{t}{2},\frac{\delta}{2t},d,R_{\mathcal{P_{X,Y}}(G_{t-1})}(f^{*})\right), \nonumber \\
			\label{t1}
			\end{eqnarray}
			and by inequality (\ref{errorbounds1}) and the definition of $\hat{f}$ we get,
			
			\begin{eqnarray}
			R_{\mathcal{P_{X,Y}}(G_{t-1})}(f^{*}) &\leq& R_{\mathcal{P_{X,Y}}(G_{t-1})}(\hat{f}) \nonumber \\ 
			&\leq& \hat{R}(\hat{f},\hat{S}) +\sigma_{R-\hat{R}}\left(\frac{t}{2},\frac{\delta}{2t},d,R_{\mathcal{P_{X,Y}}(G_{t-1})}(\hat{f}),\hat{R}(\hat{f},\hat{S})\right) \nonumber \\
			&\leq&\hat{R}(\hat{f},\hat{S}) +\hat{\sigma}_{R-\hat{R}}\left(\frac{t}{2},\frac{\delta}{2t},d,\hat{R}(\hat{f},\hat{S})\right). \nonumber \\
			\label{t2}
			\end{eqnarray}
			Plugging (\ref{t2}) into (\ref{t1}) we get,
			
			\begin{eqnarray}
			\hat{R}(f^{*},\hat{S}) &\leq&  \hat{R}(\hat{f},\hat{S}) +\hat{\sigma}_{R-\hat{R}}\left(\frac{t}{2},\frac{\delta}{2t},d,\hat{R}(\hat{f},\hat{S})\right) +\bar{\sigma}_{\hat{R}-R}\left(\frac{t}{2},\frac{\delta}{2t},d,\hat{R}(\hat{f},\hat{S}) +\hat{\sigma}_{R-\hat{R}}(\frac{t}{2},\frac{\delta}{2t},d,\hat{R}(\hat{f},\hat{S}))\right) \nonumber \\
			&&\Rightarrow f^{*}\in G_{t}.
			\label{t3}
			\end{eqnarray}
			
		\end{proof}		

		\begin{proof}[of Lemma \ref{lemma:epsilon}]
			Let $G_{t-1}$ be the final low-error set of $\ActiveiLESS$, and let $\hat{S}$ be the final set of examples. The following inequalities are derived from Lemma~\ref{lemma:errorbounds} and inequalities (\ref{t1}) and (\ref{t2}).
			\begin{eqnarray*}
				R_{\mathcal{P_{X,Y}}(G_{t-1})}(\hat{f}) 
				&\leq& \hat{R}(\hat{f},\hat{S}) +\hat{\sigma}_{R-\hat{R}}\left(\frac{t}{2},\frac{\delta}{2t},d,\hat{R}(\hat{f},\hat{S})\right) \nonumber \\
				&\leq&\hat{R}(f^{*},\hat{S}) +\hat{\sigma}_{R-\hat{R}}\left(\frac{t}{2},\frac{\delta}{2t},d,\hat{R}(\hat{f},\hat{S})\right) \nonumber \\
				&\leq& R_{\mathcal{P_{X,Y}}(G_{t-1})}(f^{*})+\bar{\sigma}_{\hat{R}-R}\left(\frac{t}{2},\frac{\delta}{2t},d,R_{\mathcal{P_{X,Y}}(G_{t-1})}(f^{*})\right)+\hat{\sigma}_{R-\hat{R}}\left(\frac{t}{2},\frac{\delta}{2t},d,\hat{R}(\hat{f},\hat{S})\right)\nonumber \\
				&\leq& R_{\mathcal{P_{X,Y}}(G_{t-1})}(f^{*})+\bar{\sigma}_{\hat{R}-R}\left(\frac{t}{2},\frac{\delta}{2t},d,\hat{R}(\hat{f},\hat{S}) +\hat{\sigma}_{R-\hat{R}}\left(\frac{t}{2},\frac{\delta}{2t},d,\hat{R}(\hat{f},\hat{S})\right) \right)  +\hat{\sigma}_{R-\hat{R}}\left(\frac{t}{2},\frac{\delta}{2t},d,\hat{R}(\hat{f},\hat{S})\right)  \nonumber \\
				&\leq& R_{\mathcal{P_{X,Y}}(G_{t-1})}(f^{*})+ \epsilon. \nonumber \\
			\end{eqnarray*}
			
			By Lemma~\ref{lemma:f*in} we know that $f^{*}$ resides within $G_{t-1}$, which implies that any change in $\mathcal{P_{X,Y}}(G_{t-1})$ in comparison to $\mathcal{P_{X,Y}}$ reduces the true error of $f^{*}$. This also means that for every $f \in \F$,
			\begin{displaymath}
			R_{\mathcal{P_{X,Y}}}(f) - R_{\mathcal{P_{X,Y}}(G_{t-1})}(f) \leq R_{\mathcal{P_{X,Y}}}(f^{*}) - R_{\mathcal{P_{X,Y}}(G_{t-1})}(f^{*}),
			\end{displaymath}
			which results in 
			\begin{displaymath}
			R_{\mathcal{P_{X,Y}}}(\hat{f})  \leq R_{\mathcal{P_{X,Y}}}(f^{*}) + \epsilon.
			\end{displaymath}
		\end{proof}		
		
		\begin{proof}[of Lemma \ref{lemma:radius_active}]
			The proof is similar to the proof of Lemma \ref{lemma:radius}. We consider the last modification of $G_t$ as a run of {\iLESS}, under $\mathcal{P_{X,Y}}(G_{t-1})$, with $m_0 \eqdef 2^{\lfloor  log_2 m \rfloor -1}$ examples and delta equal to $\frac{\delta}{4m_0}$.
			
			Under event $\cal K$, the conditions of Lemma~\ref{lemma:radius} hold, and by Lemma \ref{lemma:f*in}, $R_{\mathcal{P_{X,Y}}(G_{t-1})}(f^{*}) \leq R(f^*)$. We simply 
			apply Lemma~\ref{lemma:radius} with these parameters to get $A'$ ($A$ in 
			Lemma~\ref{lemma:radius}).
		
			$$
			A' = 4d \ln( \frac{16m_0 e}{d\delta/4m_0} )= 4d \ln( \frac{64m_0^2 e}{d\delta } ).
			$$
			The fact that $m/4 \leq m_0 \leq m/2$ completes the proof.
		\end{proof}

		\begin{proof}[of Lemma \ref{lemma:max_examples_observed}]
			We know by Lemma \ref{lemma:radius_active} that there exist constants $C_1$,$C_2$ that depend only on $\ln(\frac{1}{\delta})$ and $d$, and are independent of $m$, s.t. 
			$$
			\sigma_{Active} \leq C_1\frac{\ln m}{m} + C_2\sqrt{\frac{\ln m}{m}\cdot R(f^*)}.
			$$
			We also know by the definition of {\ActiveiLESS} (Strategy \ref{alg:Active-iLESS}), that it terminates when $\sigma_{Active}$ is smaller than the given $\epsilon$. We will find $m$ large enough s.t.	
			
			\begin{eqnarray}
			&& C_1\frac{\ln m}{m} \leq \epsilon/2, \label{bb80}\\
			&& C_2\sqrt{\frac{\ln m}{m}\cdot R(f^*)} \leq \epsilon/2 \label{bb81}.
			\end{eqnarray}	
			We assume that $\epsilon \leq 1/e$, as it is easy to find a proper $m$ for $\epsilon > 1/e$. Starting with Equation (\ref{bb80}), we want to show that $m=O(\frac{1}{\epsilon}\ln(\frac{1}{\epsilon}))$ satisfies it. Thus, we find $k_1$ s.t.
			\begin{eqnarray*}
			&& C_1\frac{\ln (k_1\frac{1}{\epsilon}\ln(\frac{1}{\epsilon}))}{k_1\frac{1}{\epsilon}\ln(\frac{1}{\epsilon})} \leq \frac{\epsilon}{2} \label{bb82}\\
			&\Leftrightarrow& \frac{\ln(k_1\frac{1}{\epsilon}\cdot \ln( \frac{1}{\epsilon}))}{\ln(\frac{1}{\epsilon})}
			\leq \frac{k_1}{2C_1}. 
			\end{eqnarray*}
			Bounding the left-hand side of the equation for $\epsilon \leq 1/e$ gives us,
			\begin{eqnarray*}	
				\frac{\ln(k_1\frac{1}{\epsilon}\cdot \ln( \frac{1}{\epsilon}))}{\ln(\frac{1}{\epsilon})} &\leq& \frac{\ln(k_1\frac{1}{\epsilon}\cdot \frac{1}{\epsilon})}{\ln(\frac{1}{\epsilon})} \\
				&\leq& 2+\ln k_1.
			\end{eqnarray*}		
			We need to find $k_1$ that will satisfy  
			\begin{eqnarray*}	
				2+\ln k_1 \leq \frac{k_1}{2C_1}.
			\end{eqnarray*}	
			$k_1=16C_1^2$ will work for $C_1 \geq 1$; otherwise, we take $k_1=10$. 
			
			We use the same procedure to show that $m=O(\frac{R(f^*)}{\epsilon^2}\ln(\frac{R(f^*)}{\epsilon^2} ))$ satisfies Equation (\ref{bb81}). We rewrite the equation in the following way:
			\begin{eqnarray*}
			&& \frac{\ln m}{m} \leq \frac{\epsilon^2}{4C_2^2 R(f^*)} \eqdef \epsilon_0.
			\end{eqnarray*}
			We assume that $\epsilon_0 \leq 1/e$ ($m=4$ holds otherwise) and find $k_2$ s.t. 
			\begin{eqnarray*}
				&& \frac{\ln (k_2\frac{1}{\epsilon_0}\ln(\frac{1}{\epsilon_0}))}{k_2\frac{1}{\epsilon_0}\ln(\frac{1}{\epsilon_0})} \leq \epsilon_0.
			\end{eqnarray*}
			As before, we reduce the problem to finding $k_2$ that satisfies
			\begin{eqnarray*}
				2+\ln(k_2) \leq k_2.
			\end{eqnarray*}
			$k_2 = 4$ suffices. We thus get that $m=O(\frac{1}{\epsilon_0^2}\ln(\frac{1}{\epsilon_0^2} ))=O(\frac{R(f^*)}{\epsilon^2}\ln(\frac{R(f^*)}{\epsilon^2} ))$ satisfies Equation (\ref{bb81}). This implies that there exists a function $m(1/\epsilon, R(f^*)) = O \left(\frac{1}{\epsilon}\ln(\frac{1}{\epsilon})+ \frac{R(f^*)}{\epsilon^2}\ln(\frac{R(f^*)}{\epsilon^2} \right) $ that bounds the total number of labels processed by {\ActiveiLESS}.
		\end{proof}
		
		\begin{lemma}
			\label{ChernoffBound}
			\cite{JohnCannyNotes} Let $X_1,X_2,...,X_n$ be independent Bernoulli trials with $Pr[X_i=1]=p$, let $X\eqdef\sum_{i=1}^{n} X_i$, and let $\mu = \E X$. Then, for every $\alpha \geq 0$: $$\Pr (X<(1-\alpha)\mu) \leq \exp(-\mu \alpha^2 /2).$$
		\end{lemma}	
		
		\begin{proof}
			This proof is taken from the work of John Canny \cite{JohnCannyNotes}.
			
			For $t>0$, we have 
			\begin{eqnarray}
				\Pr (X<(1-\alpha)\mu) = \Pr (\exp(-tX) > \exp(-t(1-\alpha)\mu)). \label{ap1}
			\end{eqnarray}
			We use Markov's inequality. For a nonnegative random variable $X$, and $a>0$,
			\begin{eqnarray*}
				\Pr (X \leq a) \leq \frac{\E(X)}{a}.
			\end{eqnarray*}
			We apply the inequality for the right-hand side of Equation (\ref{ap1}), to get
			\begin{eqnarray}
				\Pr (X<(1-\alpha)\mu) \leq  \frac{\E(\exp(-tX))}{\exp(-t(1-\alpha)\mu)}. \label{ap2}
			\end{eqnarray}
			$X_1,X_2,...,X_n$ are independent and thus
			$$\E(\exp(-tX)) = \prod_{i=1}^{n} \E(\exp(-tX_i)).$$
			For each $X_i$
			$$\E(\exp(-tX_i)) = pe^{-t} + (1-p) = 1 - p(1 - e^{-t}).$$
			We use the fact that $1-x < \exp(-x)$ for all $x$, with $x=p(1 - e^{-t})$, to get 
			$$\E(\exp(-tX_i)) \leq \exp(-p(1 - e^{-t})), $$
			and conclude that 
			\begin{eqnarray} \label{bb60}
			\E(\exp(-tX)) = \prod_{i=1}^{n} \E\left(\exp(-tX_i)\right) \leq \prod_{i=1}^{n} \exp\left(-p(1 - e^{-t})\right)  \nonumber\\
			 =\exp\left(\sum_{i=1}^{n} p(e^{-t}-1)\right) = \exp\left(\mu (e^{-t}-1)\right).
			\end{eqnarray}
			Going back to Equation (\ref{ap2}), we have,
			\begin{eqnarray}
				\Pr (X<(1-\alpha)\mu) \leq \frac{\exp \left( \mu (e^{-t}-1)\right)}{\exp\left(-t(1-\alpha)\mu\right)} = \exp\left(\mu (e^{-t}-1+t-t\alpha ) \right). \label{ap3}
			\end{eqnarray}
			We choose $t>0$ to make the right-hand side of the equation as small as possible. After derivation, we get that the best $t$ is $t=\ln(\frac{1}{1-\alpha})$, and plugging it into Equation (\ref{ap3}) gives us,
			\begin{eqnarray}
			\Pr (X<(1-\alpha)\mu) &\leq& \exp\left(\mu (1-\alpha-1+\ln(\frac{1}{1-\alpha})-\ln(\frac{1}{1-\alpha})\alpha )\right) \nonumber \\
								  &=& \exp \left( \mu (-\alpha+\ln(\frac{1}{1-\alpha})(1-\alpha) ) \right) \nonumber \\
								  &=& \left(\frac{e^{-\alpha}}{(1-\alpha)^{1-\alpha}}\right)^\mu.	 \label{ap4}
			\end{eqnarray}			
			We now simplify this bound to get the desired result. We know that $(1-\alpha)^{1-\alpha} = e^{(1-\alpha)ln(1-\alpha)}$, and by Taylor expansion 
			$$
			ln(1-\alpha) = -\alpha - \frac{\alpha^2}{2} - \frac{\alpha^3}{3}... ,
			$$
			which multiplied by $(1-\alpha)$, gives us
			\begin{equation}
			(1-\alpha)ln(1-\alpha) = -\alpha + \frac{\alpha^2}{2} + \text{positive terms} > -\alpha + \frac{\alpha^2}{2}.	\label{cccc}
			\end{equation}
			Plugging (\ref{cccc}) into Equation (\ref{ap4}), we finally get,
			\begin{eqnarray}
			\Pr (X<(1-\alpha)\mu) &\leq& \left(\frac{e^{-\alpha}}{(1-\alpha)^{1-\alpha}}\right)^\mu \nonumber \\ 
								  &=& \left(\frac{e^{-\alpha}}{e^{(1-\alpha)ln(1-\alpha)}}\right)^\mu \nonumber \\ 
								  &\leq& \left(\frac{e^{-\alpha}}{e^{-\alpha + \frac{\alpha^2}{2}}}\right)^\mu \nonumber \\ 
								  &=& e^{-\mu\alpha^{2}/2}
			\end{eqnarray}
		\end{proof}
		
		\begin{lemma}
			\label{ChernoffBound2}
			\cite{JohnCannyNotes} Let $X_1,X_2,...,X_n$ be independent Poisson trials with $Pr[X_i=1]=p$, let $X\eqdef\sum_{i=1}^{n} X_i$, and let $\mu = \E X$. Then, for every $\alpha\geq 2e-1$: 
			$$
			\Pr (X>(1+\alpha)\mu) \leq 2^{-\mu \alpha}.
			$$\\
		\end{lemma}	
		\begin{proof sketch}
			This sketch is taken from the work of John Canny \cite{JohnCannyNotes}. It is almost identical to the proof of Lemma \ref{ChernoffBound}.
			
			We start by showing that 
		\begin{eqnarray*}
			\Pr (X>(1+\alpha)\mu) &\leq& \left(\frac{e^{\alpha}}{(1+\alpha)^{1+\alpha}}\right)^\mu.	
		\end{eqnarray*}	
		For every $t>0$, 
		\begin{eqnarray*}
		\Pr (X>(1+\alpha)\mu) &=& \Pr[\exp(tX)>\exp \left( t(1+\alpha)\mu \right) ].	
		\end{eqnarray*}
		As we did in Lemma \ref{ChernoffBound}, we compute the Markov bound, 
		
		\begin{eqnarray*}
		\Pr (X>(1+\alpha)\mu) &\leq& \frac{\E(\exp(tX))}{\exp(t(1+\alpha)\mu)},
		\end{eqnarray*}
		and use the fact that $X_i$ are independent, just like in (\ref{bb60}), to get that
		$$\E(\exp(tX)) \leq \exp \left( \mu (e^{t}-1) \right) .$$
		Thus we get that 
		\begin{eqnarray*}
		\Pr (X>(1+\alpha)\mu) &\leq& \frac{\exp(\mu (e^{t}-1))}{\exp(t(1+\alpha)\mu)}=\exp\left(\mu(e^t-1-t-\alpha t)\right).
		\end{eqnarray*}
		From deviation, we choose $t=\ln(1+\alpha)$ to get
		\begin{eqnarray*}
			\Pr (X>(1+\alpha)\mu) &\leq& \left(\frac{e^{\alpha}}{(1+\alpha)^{1+\alpha}}\right)^\mu.	
		\end{eqnarray*}	
		For $\alpha \geq 2e-1$:
		\begin{eqnarray*}
			\Pr (X>(1+\alpha)\mu) &\leq& \left(\frac{e^{\alpha}}{(1+\alpha)^{1+\alpha}}\right)^\mu \leq \left(\frac{e^{\alpha}}{(2e)^{1+\alpha}}\right)^\mu \leq \left(\frac{e^{\alpha}}{(2e)^{\alpha}}\right)^\mu = 2^{-\mu\alpha}.	
		\end{eqnarray*}	
		\end{proof sketch}
		
		\begin{lemma}
			\label{lemma:radius_batch}
			Given that event $\cal K$ (see Definition \ref{K}) occurred, the radius of $\BatchiLESS$, as defined in Strategy \ref{alg:Active-iLESS}, stage 4, satisfies
			\begin{eqnarray}
			\sigma_{Active}=O\left(\frac{B}{m}+ \sqrt{\frac{B}{m} \cdot R(f^{*})} \right),
			\end{eqnarray}
			where $B \eqdef 4d \ln( \frac{8m^2e}{d\delta} ).$
		\end{lemma}
		
		\begin{proof}
			{\BatchiLESS} simulates a run of {\ActiveiLESS}. 
			Consider a run of \ActiveiLESS with $m_0$ examples and $\delta = \delta_0$. The last iteration in which $G_t$ has changed (relative to $G_{t-1}$) was iteration $2^{\lfloor \log_2 m_0 \rfloor} \eqdef T$. $G_T$ is calculated in exactly the same way as {\iLESS} calculates its $G$ under probability distribution $\mathcal{P_{X,Y}}(G_{T-1})$, when it
			is provided with $T/2$ examples, and $\frac{\delta_0}{2T}$ as its delta. 
			Assuming that event $\cal K$ occurred, we deduce that event $\cal E$ (see Definition \ref{E}) occurred as well. Therefore,  Lemma~\ref{lemma:radius} holds for the last iteration of  	{\BatchiLESS}.
			 
		{\iLESS} operates in this run on  $T/2$ labeled examples, and it holds that $m_0/4 \leq T/2 \leq m_0/2$. 
		The delta it uses in this run is  $\frac{\delta_0}{2T} > \frac{\delta_0}{m_0}$, so by Lemma \ref{lemma:radius},
		we have
			\begin{equation*}
			\sigma_{Active}	\leq 6\frac{B}{m_0/4} + 3\sqrt{\frac{B}{m_0/4} \cdot R_{\mathcal{P_{X,Y}}(G_{t-1})}(f^{*})} = 24\frac{B}{m_0} + 6\sqrt{\frac{B}{m_0} \cdot R_{\mathcal{P_{X,Y}}(G_{t-1})}(f^{*})}.
			\end{equation*}
			To finish the proof, we need to show that $R(f^*) \geq R_{\mathcal{P_{X,Y}}(G_{t-1})}(f^{*})$.  From Lemma \ref{lemma:f*in}, we know that when $\cal K$ occurs, any $f^{*}$ of the original distribution $\mathcal{P_{X,Y}}$ resides within $G_{t}$ for all $t$. Thus, the true error of $f^*$ can only decrease under the revised distribution 
			$G_{t-1}(f^{*})$.
		\end{proof}

		\begin{theorem}  
			\label{thm:LessRejection_batch}
			Let $\mathcal{F}$ be a hypothesis class with VC-dimension $d$, and let	
			$\mathcal{P_{X,Y}}$ be an unknown probability distribution. Assume that event $\cal K$ (see Definition \ref{K}) occurred. Then, for all $f^*$, the abstain rate is bounded by
			$$
			1-\Phi(\BatchiLESS) \leq \theta_{f^*}(R_{0}) \cdot R_{0},	
			$$		
			where
			$$
			R_{0} \eqdef 2\cdot R(f^{*}) + 44 \cdot \frac{B}{m} + 12 \cdot \sqrt{\frac{B}{m} \cdot R(f^{*})}.
			$$
			where $B \eqdef 4d \ln( \frac{8m^2e}{d\delta} ).$
			This immediately implies (by definition) that 
			$$
			1-\Phi(\BatchiLESS) \leq \theta(R_{0}) \cdot R_{0}.	
			$$	
		\end{theorem} 
		
		\begin{proof sketch}
			The proof is very similar to the proof of Lemma \ref{lemma:radius_batch}. We observe the last modification of $G_T$, and notice that the change was made according to a run of {\iLESS}, on the implied probability distribution $\mathcal{P_{X,Y}}(G_{T-1})$. Then we simply activate Theorem \ref{thm:LessRejection} with the relevant parameters plugged into it. 
			
			Note that by Lemma \ref{lemma:f*in}, all $f^*$ of the original distribution reside within $G_t$ for all $t$, and thus, by Lemma \ref{lemma:f_star_best}, they are all true risk minimizers of $\mathcal{P_{X,Y}}(G_{T-1})$. This also implies that $R(f^*) \geq R_{\mathcal{P_{X,Y}}(G_{t-1})}(f^{*})$ and thus can be used to bound Equation (\ref{eq30}) of the original theorem that was proven for {\LESS}. $\theta_{f}$ is independent of $\mathcal{P_{Y|X}}$ for all $f$, and thus the change of the labels does not affect it.
			
		\end{proof sketch}

		
		\appendix

		\bibliography{S_PAC}
	\end{document}